\documentclass{article}
\usepackage{hyperref}

\usepackage{microtype}
\usepackage{graphicx}
\usepackage{subfigure}
\usepackage{booktabs} % for professional tables

\usepackage{microtype}
\usepackage{graphicx}
\usepackage{subfigure}
\usepackage{booktabs}
\usepackage{cuted}

\usepackage{amsfonts}
\usepackage{caption}
\usepackage{amsmath}
\usepackage{amssymb}
\usepackage{mathtools}
\usepackage{amsthm}
\usepackage{relsize}
\usepackage{thmtools} 
\usepackage{thm-restate}

\usepackage{float}
\usepackage{tabularx}
\usepackage{array}
\usepackage{enumerate}
\usepackage{courier}
\let\oldtexttt\texttt
\renewcommand{\texttt}[1]{{\footnotesize\oldtexttt{#1}}}

\usepackage[capitalize,noabbrev]{cleveref}

\usepackage[textsize=tiny]{todonotes}

\usepackage{amsfonts}       % blackboard math symbols
\usepackage{amsmath}
\usepackage{amsthm}
\usepackage{amssymb}
\usepackage[mathscr]{euscript}

\usepackage{nicefrac}       % compact symbols for 1/2, etc.
\usepackage{microtype}      % microtypography

\usepackage{tikz}
\usepackage{adjustbox}

\newcommand{\Rset}{\mathbb R}

\newcommand{\ignore}[1]{}

\newtheorem{corollary}{Corollary}

\newcommand{\fusion}{\textsc{Fusion}}
\newcommand{\F}{\mathcal{F}}

\usepackage{enumitem}
\usepackage{ulem}
\usepackage{color}

\usepackage[accepted]{icml2024}

\definecolor{Darkblue}{rgb}{0,0,0.4}
\definecolor{Brown}{cmyk}{0,0.81,1.,0.60}
\definecolor{Purple}{cmyk}{0.45,0.86,0,0}

\hypersetup{colorlinks=true,%pdfborder={1 1 1 [3]},%
           citebordercolor={.6 .6 .6},linkbordercolor={.6 .6 .6},%
citecolor=Darkblue,urlcolor=black,linkcolor=red}
\newcommand{\lref}[2][]{\hyperref[#2]{#1~\ref*{#2}}}

\setlist[enumerate]{itemsep=0.5pt,parsep=0pt,before={\parskip=0pt},leftmargin=0.5cm}
\setlist[itemize]{itemsep=1.0pt,parsep=1.5pt,before={\parskip=1.5pt},leftmargin=0.5cm}

\begin{document}

\icmltitlerunning{Deep Fusion}

\twocolumn[

\icmltitle{Deep Fusion: Efficient Network Training via Pre-trained Initializations}

\begin{icmlauthorlist}
\icmlauthor{Hanna Mazzawi}{goog}
\icmlauthor{Xavi Gonzalvo}{goog}
\icmlauthor{Michael Wunder}{goog}
\icmlauthor{Sammy Jerome}{goog}
\icmlauthor{Benoit Dherin}{cloud}
\end{icmlauthorlist}

\icmlaffiliation{goog}{Google Research, New York, NY, USA}
\icmlaffiliation{cloud}{Google, Sunnyvale, CA, USA}

\icmlcorrespondingauthor{Hanna Mazzawi}{mazzawi@google.com}

\vskip 0.3in
]
\printAffiliationsAndNotice{}

\begin{abstract}
In recent years, deep learning has made remarkable progress in a wide range of domains, with a particularly notable impact on natural language processing tasks. One of the challenges associated with training deep neural networks in the context of LLMs is the need for large amounts of computational resources and time. To mitigate this, network growing algorithms offer potential cost savings, but their underlying mechanisms are poorly understood.
We present two notable contributions in this paper. First, we present Deep Fusion, an efficient approach to network training that leverages pre-trained initializations of smaller networks. Second, we propose a theoretical framework using backward error analysis to illustrate the dynamics of mid-training network growth.
Our experiments show how Deep Fusion is a practical and effective approach that not only accelerates the training process but also reduces computational requirements, maintaining or surpassing traditional training methods' performance in various NLP tasks and T5 model sizes.
Finally, we validate our theoretical framework, which guides the optimal use of Deep Fusion, showing that with carefully optimized training dynamics, it significantly reduces both training time and resource consumption.
\end{abstract}

\section{Introduction}
% discussion about model size, larger is better. So there's a need to train very large models.
Large language models (LLMs) have significantly advanced the state of the art in various natural language processing  tasks, including text generation, translation, summarization, and question answering. However, training these models demands substantial amounts of data and computational resources. As a result, there has been a growing interest in developing efficient training methods to address the challenges associated with the high computational costs and energy consumption during the training process \cite{Narayanan21large_scale_llm}.

% discussion on how to train larger models. First transformers are ideal for this. Then explain difference between data and model parallelism.
While some studies \cite{kaplan2020scaling, touvron2023llama, zhou2023lima} discuss that a balance of data and model size is important, it's undeniable that larger models often yield better performance \cite{chowdhery2022palm}. Several experiments and publications have demonstrated that as model size increases, the performance on various natural language processing tasks continues to improve \cite{devlin2018bert, radford2019language, brown2020language}. This trend is evident in the progression of LLMs, such as BERT, \mbox{GPT-2}, \mbox{GPT-3}, and PaLM where each successive generation is larger and achieves better results across a wide range of benchmarks~\cite{gu2021palm}.

Advancements in LLM efficiency have been driven by a variety of innovative techniques that enable faster training or inference without sacrificing performance. One such approach is model compression, which has been shown to reduce LLM size without significant loss in accuracy \cite{Ganesh21compress,zhang2022platon,kwon2022a}. Similarly, adaptive computation time methods have been proposed to dynamically allocate computational resources during LLM training, leading to improved efficiency \cite{graves2016adaptive}. Techniques such as layer-wise adaptive rate scaling (LARS) and layer-wise adaptive rate control (LARC) have demonstrated accelerated convergence in LLMs by adapting learning rates on a per-layer basis \cite{you2017scaling, you2017large}. Moreover, recent studies have explored the potential of mixed-precision training, where lower-precision computation is employed during the training process to speed up training and reduce memory requirements \cite{micikevicius2017mixed}.
One technique attracting more interest lately is network growing, which is based on increasing transformer size  during training. This approach has been shown to speed up training in various studies~\cite{shen2022staged,bert_pstacking,twice_faster,transformer_grow_p_bert,100_k_dollars}, while others report similar or even lower performance compared to standard training methods~\cite{refute_growing}. However, these methods are not well understood and add complexity to the already challenging process of training large language models, both in terms of implementation and correct deployment.

In addition to efficiency, a critical aspect to scale LLM training is distributed training, typically done with a mix of data and model parallelization. The former splits the training batch across accelerators (e.g., GPUs), while the latter splits the model operations across accelerators so that each accelerator computes part of the model.
While data parallelism alone is typically the easiest to implement, it is not well suited for very large models as it needs the whole model to fit in a single accelerator.  Model parallelism can be efficient, but it can be more difficult to implement as the dependency between accelerators’ input and outputs can lead to degraded performance.
%
% 1. Deep fusion
In our research, we present two novel contributions to emphasize network growing for training efficiency as a primary goal. First, we introduce Deep Fusion, an efficient approach to network training that leverages pre-trained initializations of smaller networks. We employ a fusion operator to combine these smaller networks, promoting wide over-parameterization as a growing mechanism.
Deep Fusion serves as a new method to distribute training across accelerators. By enabling the growth from these smaller models, it allows for a modular and efficient distribution of parallel training before the fusion operation, effectively training smaller models for an initial portion of the training process.

% 2. BEA
Second, we introduce a new theoretical framework based on Backward Error Analysis (BEA) \cite{hairer2006} to analyze the implicit regularization the optimization process introduces during training right after the fusion operation takes place.
BEA has been shown to be an essential analysis tool in many settings, continual-learning~\cite{dherin2023implicit} being of special interest for this work.
The main result of the BEA framework is the decomposition of a modified loss into the original loss plus terms that are implicitly minimized during training. When we apply this analysis to Deep Fusion, it becomes clear that the key aspect of our approach lies in a specific interaction component, often referred to as a `Lie bracket' \cite{lee2012smooth_manifolds} in differential geometry. This component encapsulates the dynamic interplay between the gradients of the smaller and larger networks, highlighting how they influence and adjust to each other during the first step of the training process.

As we will see throughout the paper, Deep Fusion demonstrates significant reductions in both training time and resource consumption, while maintaining or improving generalization performance.

\subsection{Contribution}
As part of our research, this paper proposes:
\begin{itemize}
  %\item Better than random parameter initialization of large models from smaller already trained ones.
  \item A method that focuses on initializing large networks from training smaller networks, and employing fusion operators to combine them. This method promotes wide over-parameterization, which leads to improved efficiency in network training.
  %\item Clear modular and independent sub-training tasks with autonomy to allow for usage on non-linked scattered compute for part of the training.
  
  \item An effective framework for the utilization of data and model parallelization techniques, as well as the strategic use of accelerator devices to train models of smaller size. This allows our approach to significantly reduce training time while increasing the performance of the resulting networks.

  %\item Improved model knowledge performance that stems from ability to cover more data while training.
  %\item Faster training of language models (pretrain) from regular distributed training by replicating trained parameters for portions of the network.
  \item A new theoretical framework to analyze dynamically growing wider network algorithms, backed with validation experiments.

  \item A downstream task evaluation with LLMs, demonstrating its effectiveness and efficiency in various scenarios.
\end{itemize}

%The paper is organized as follows: Section \ref{related} hold related works, Section \ref{fusion} present the core of the Fusion technique. Section \ref{experiments} hold experimentation results. Finally, Section \ref{discussion} present a discussion and conclusions.

\section{Related Work}\label{related}

In line with the lottery ticket hypothesis \cite{frankle2018lottery, frankle2019lottery}, our work shares the following belief: The most commonly used initialization schemes, primarily discovered heuristically \cite{glorot2010understanding, he2015delving}, are sub-optimal. %and there is significant room for improvement.
While there's some evidence that over-parameterization may not be necessary during training \cite{nakkiran2020deep, belkin2019reconciling}, we still believe over-parameterization and a ``good'' initialization can yield better performance. Thus, we aim to actualize some of the potential that comes from finding a more principled initialization scheme.

%These findings suggest the potential for developing a more principled initialization scheme. However, the process of finding winning tickets involves computationally expensive cycles of training (potentially large) models from scratch to convergence and pruning.

From a transfer learning perspective, progressive networks \cite{rusu2016progressive} grow networks to address the problem of forgetting previous tasks. Another approach, deep model consolidation \cite{zhang2020classincremental}, uses a smaller pre-trained model to provide a better initialization for a larger model, which is then fine-tuned on a new task. Network morphism \cite{wei2016network} is another approach that aims to find a larger network by transforming a smaller, pretrained network while preserving the network function during the transformation. This is achieved by expanding the original network with layer-wise operations that preserve input-output behavior.
%such as adding neurons or layers, and adjusting the weights accordingly to ensure that the expanded network maintains the same input-output behavior as the smaller network. %The larger network can then be fine-tuned for the target task, potentially benefiting from the knowledge already acquired by the smaller network during its initial training.

Similar to our method, staged training \cite{shen2022staged} also focuses on network efficiency. This approach involves defining a growth operator while preserving constraints associated with loss and training dynamics. By gradually expanding the model capacity, 
%and ensuring that the growth process adheres to specific constraints, 
staged training allows for more efficient training. We argue that preserving training dynamics might not be the most effective approach when it comes to fusion. In fact, it could be counterproductive, and exploring high learning rate cycles could offer a preferable alternative. We validate this theoretically and show it empirically. %\deletetext{Furthermore, we enhance the fuse operator by developing a more efficient initialization for cross-connections.}

BEA has been used in many settings to uncover inductive training biases of various optimizers (e.g. gradient descent \cite{igr}, SGD \cite{smith2021on}; momentum \cite{ghosh2023implicit}; Adam and RMSProp \cite{cattaneo2023implicit_bias_adam}), various architectures (e.g., GANs \cite{rosca2021discretization}; diffusion models \cite{gao2023diffusion_bea}), and various training settings like continual-learning \cite{dherin2023implicit} or distributed and federated learning \cite{barba2021federated}.

% \section{Notation}\label{notation}
% Let $x$ be a vector in $\mathbb{R}^n$ and $y$ be a vector in $\mathbb{R}^m$. We denote by $x||y$ the vector in $\mathbb{R}^{n+m}$ that is the results of concatenating $x$ and $y$.

% Let $f$ be a function $f: \mathbb{R}^n \to \mathbb{R}^m$, and let $g$ be a function $g: \mathbb{R}^m \to \mathbb{R}^k$. We denote by $f\circ g$ the composition of $f$ and $g$. That is,
% $$
% f\circ g (x) = f(g(x)).
% $$
% Let $M$ be a machine learning model from $\mathbb{R}^n\to \mathbb{R}^m$. We denote by $\mathcal{L} (M)$ the set of layer or operation functions $\{\ell_1, \ell_2, \ldots, \ell_p\}$ that compute the model $M$. In other words,
% $$
% \mathcal{M}(x) = \ell_1 \circ \ell_2 \circ \cdots \circ \ell_p (x).
% $$
% For two functions $\ell$ and $q$ in $\mathbb{R}^{n_1} \to \mathbb{R}^{m_1}$ and $\mathbb{R}^{n_2} \to \mathbb{R}^{m_2}$, we denote by $\ell || q$ the function in $\mathbb{R}^{n_1+n_2} \to \mathbb{R}^{m_1+m_2}$ defined as follows:
% $$
% \ell || p (x || y) = \ell(x) || p(y),
% $$
% for any $x\in\mathbb{R}^{n_1}$ and $y\in\mathbb{R}^{n_2}$.

% Finally, for two function $\ell$ and $p$ as above, we denote by $Avg(f,p)$ their average function, that is,
% $$
% Avg(f,p)(x||y) = (f(x) + p(y))/2.
% $$

%----------------------------------------------------------------------------------------------

\section{Fusion}\label{fusion}
We start by demonstrating our \fusion\ operator on two fully connected layers before expanding to T5 transformers.

%In this section we demonstrate our \fusion\ operator principles on two fully connected networks, and then expand to a full T5 transformer.

%We consider a standard supervised learning scenario and assume that both training and test points are drawn independently and identically distributed (i.i.d.) according to some distribution $\sD$ over $\sX \times \sY$. Each $(x_i, y_i)$ pair sample from $S = \{(x_1, y_1), \ldots ,(x_m, y_m) \}$ represents a feature vector and its corresponding class label, respectively.

A generic neural network is a function $f \colon \Rset^d \to \Rset^k$ defined with $r$ layers. A layer $g_k: \Rset^{n} \to \Rset^{m}$ has weights in layer $k \in [r]$ $w_k \in \Rset^{n \times m}$ and biases being $b_k \in \Rset^{m}$. That is, for each layer $k$ we calculate
\begin{equation}
\label{eq:layer_architecture}
    a_k = g_k(a_{k-1}) = \phi_k( a_{k-1}w_k + b_k),
\end{equation}
where $a_0 = x$ is the input vector, and $\phi_k$ is the $k$-th activation function. 
In what follows, we will omit $a_k$ when it is clear from context.
%In what follows, when $a_k$ is clear from context we will omit it.

The output of the neural network is defined as the composition of the $r$ layers,
\begin{equation}
    \label{eq:composition}
    f(x) = g_r \circ \ldots \circ g_2  \circ g_1(x).
\end{equation}

Our \fusion\ operator $\F$ takes two layers from two different models and generates a new layer by
composing their weights and biases. The fused layer has two characteristics:
\begin{itemize}
    \item \textbf{Fusion rule}: the fused layer maintains the same composition or architecture defined in Eq.\ref{eq:layer_architecture}. That is, we do not allow a change in the architecture, but rather a change in the dimensionality of the operations.
    \item  \textbf{Fusion property}: the fused layer calculates the concatenation of the the two original layers that are fused.
    %the fusion operation is defined as a concatenation operator, so given vectors $v$ and $v'$ the new fused vector $v^{(f)}=[v,v']$.
\end{itemize}

The fusion operator ($\F$) is defined as follows. Given two layers with $n$, $n'$ inputs and $m$, $m'$ outputs,
\begin{align}
    \F_w & \colon \Rset^{n\times m} \times \Rset^{n'\times m'} \to \Rset^{(n+n')\times(m+m')}, \\
    \F_b & \colon \Rset^{m} \times \Rset^{m'} \to \Rset^{(m+m')}.
\end{align}
The \fusion\ of the weights performed by $\F_w$ results in a new matrix where the weights of the layers
of the two models are located in the diagonal and the rest is set to zero. Similarly, the new bias is
simply the concatenation of the bias of the two layers being fused. So the new fused weight $w^{(f)}$ and new bias $b^{(f)}$ taking the weights of two layers,~$w$,~$w'$, and bias $b$, $b'$, respectively is defined as,
\begin{equation}
    w^{(f)} = \begin{pmatrix}
    w & \mathbf{0} \\
    \mathbf{0} & w'
    \end{pmatrix}, \quad\quad b^{(f)} = [b, b'],
    \label{eq:diagonal}
\end{equation}
where $\mathbf{0}$ is the zero matrix.
The output of the fused layer $k$ is defined as,
\begin{align*}
    %\label{eq:layer_fusion}
    g_k^{(f)} = & \F(g_k,g_k') = \phi_k(a_{k-1}^{(f)} \F_w(w_k,w_k') + \F_b(b_k,b_k')) \\\nonumber
    = & \phi_k\left([a_{k-1},a_{k-1}']\begin{pmatrix}
    w_k & \mathbf{0} \\ 
    \mathbf{0} & w_k'
    \end{pmatrix}  + [b_k,b_k']\right) \\\nonumber
    %= & \phi_k([a_{k-1}w_k,a_{k-1}'w_k'] + [b_k,b_k']) \\\nonumber
    = & \phi_k([a_{k-1}w_k+b_k,a_{k-1}'w_k'+b_k'])=  [g_k,g_k']. %\\\nonumber
    %= & [h_k,h_k'].
\end{align*}
This means that the result of the \fusion\ operator on two layers is the concatenation of the outputs, that is $[h_k,h_k']$.

\subsection{Deep Fusion and Self Deep Fusion}
For two neural networks $f$ and $f'$ defined as in Eq.~\ref{eq:composition}, the deep fusion of the two models is defined as follows: We first extend the notation of $\F$ as follows,
%
%\begin{align*}
 $$   \F(f, f') = \F(g_{r}, g_{r}') \circ \ldots  \circ \F(g_1, g_1')([x,x]);$$
%\end{align*}
And we denote by 
$
\textsc{Avg}(x,y) = (x+y)/2.
$
The function that averages two vectors of the same dimension, then the deep fusion is defined as,
\begin{equation*}
    DF(f,f') = \textsc{Avg}(\F(f, f')).
\end{equation*}

Intuitively, the deep fused model is maintaining a concatenation of the hidden representations from models $f$ and $f'$ (fusion property) throughout the network, and taking the average of their logits.

This means that after the deep fusion operation, the function calculated by the model is equivalent to the function of average ensemble of the two models. However, if we continue training the fused model, the extra parameters added by the zero blocks in the example can start leveraging the hidden representation from the cross model, and potentially lead to better performance.

Deep fusion allows the models to be distributed across multiple GPUs while still taking advantage of the strengths of both data parallelism and model parallelism.

Self deep fusion of a model $f$ is defined as deep fusing the model with itself (that is, $DF(f, f)$). It can be considered a loss preserving growth operation that does not change the network's predictions to any given input. %In what follows, we omit the word deep if it is clear from the context.

\subsection{Deep Fusing T5 Transformers}
This section describes how to deep fuse two (or more) T5 models \cite{t5paper}, $f$ and $f'$, discussing the particularities of each layer type. Once the fusion is completed, the hidden representation of the newly fused model should be a combination of the two hidden representations from the original models, aligned along the feature dimension axis.

Starting from the bottom layer, the fusion of the embedding layer is achieved by simply concatenating the smaller embeddings together (on the feature dimension axis). Next, for the multi-head attention, if $f$ has $y$ heads, and $f'$ has~$y'$ heads, then, the fused model will have $y~+~y'$ heads, and the additional heads do not interact with the others. 
All projections (query, key, value, attention output) as well as the MLP blocks are converted to diagonal matrices with extra parameters initialized to zero (see eq.~(\ref{eq:diagonal})) to prevent leaking information from the wrong hidden representation at initialization.

%We start with the embedding layer. To fuse the two embedding layers, we simply concatenate their lookup tables on the feature dimension axis. Next, for the multi-head attention, if $M_1$ has $y_1$ heads, and $M_2$ has $y_2$ heads, then, the fused model will have $y_1~+~y_2$ corresponding heads. 

%All projections: query, key, value, attention output as well as the MLP blocks are converted to diagonal matrices with extra parameters initialized to zero (similar to eq.~(\ref{eq:diagonal})) to prevent leaking information from the wrong hidden representation at initialization. 

%All projections query, key, value projections are converted to diagonal matrices with extra parameters initialized to zero (similar to eq. (\ref{eq:diagonal})) to prevent leaking information from the wrong hidden representation at initialization. The output projection from the multi-head output is treated in a similar way. Similarly, the MLP blocks are treated exactly the way we explained a fully connected layer is fused.

Note that skip connections and activations are parameter free and do not need further handling. Similarly, the element-wise scaling operation holds a scaling parameter per element in the hidden representation, and thus is trivial to fuse. %therefore, fusing this operation is trivial, we simply concatenate the scaling vectors.

Lastly, the fusion of the normalization of the hidden representation between attention and MLP layers proves to be unfeasible. This is due to the fact that it's not possible to uphold the fusion rule and the fusion property simultaneously. %In other words, we cannot concatenate two normalized vectors and still stay true to the normalization operation.
%Unlike, operations such as activation, skip-connection, element scaling that are element-wise operation and therefore trivial to fuse, the normalization of the hidden representation between attention and MLPs layers is impossible to fuse. I.e., we cannot maintain the fusion rule together with the fusion property. To put it simply, we cannot have a concatenation of two normalized vectors and also stay true to the normalization operation. 
For the normalization layer we either: 1)~Preserve the fusion property but break the fusion rule by normalizing the hidden representations of the sub-models individually and then concatenating them; or 2) keep the fusion rule but violate the fusion property by treating the concatenated hidden representation as a single vector for normalization. It's important to note that the first option requires additional coding beyond parameter initialization, unlike the second option. This dilemma doesn't occur in self deep fusion.
%As you will see in the sequel, option 1) leads to better performance, but requires coding beyond the parameter initialization, where option 2) is simpler and does not require architecture changes, but it slightly suffers in performance in comparison. 

% \subsection{Self Deep Fusion}\label{self_fusion}
% Self deep fusion of a Model $M$ is defined as deep fusing the model with itself. Note that,
% \begin{itemize}
%     \item Self deep fusion can be thought of as a growth operation (without changing its predictions to any given inputs).
%     \item In self deep fusion, we can maintain the fusion rule, and fusion property in the normalization layer.
% \end{itemize}

%After explaining the various algorithms, we are ready to present experiments with Fusion.

%\section{Gradient Flow Bias}\label{theory}
%In this section, we introduce a new theoretical framework for analyzing our algorithm. This framework can be easily generalized to network-growing algorithms and aims to explain the biases introduced to the learning that results from the growth operation. We are specifically interested in understanding the implicit bias for the first immediate gradient  descent  step after fusion.

\section{Gradient Flow Bias}\label{theory}
In this section, we analyze the inductive bias introduced by network growth using BEA \cite{hairer2006}. The idea behind BEA is to approximate the discrete updates of the optimizer by a continuous gradient flow (described by an ODE called the {\it{modified equation}}) on a modified loss that comprises the original loss and additional terms modelling the inductive biases. For instance, \cite{igr} showed that a single update of SGD  follows a gradient flow on the modified (batch) loss $\tilde L(\theta) = L(\theta) + \frac h4 \|\nabla_{\theta} L(\theta)\|^2$, whose inductive bias is the flatness regularizer $ \frac h4 \|\nabla_{\theta} L(\theta)\|^2$ depending on the learning rate $h$. More precisely, the modified equation is $\dot \theta(t) = -\nabla \tilde L(\theta(t)$ and its solution $\tilde \theta(h)$ starting at $\theta$ and evaluated at $t=h$ approximates the SGD step $\theta' = \theta - h\nabla L(\theta)$ with an error of $\mathcal O(h^3)$. The case of two consecutive updates on the batch losses $L_1$ and $L_2$ has recently been worked out in \cite{dherin2023implicit} in the context of continual learning with batches of changing data distribution. In this case, the modified equation is $\dot \theta(t) = -\nabla \tilde L(\theta(t)) - \frac h2 [\nabla L_1, \nabla L_2](\theta(t))$ with modified loss $\tilde L(\theta) = L_1(\theta) + L_2(\theta) + \frac h2 \|\nabla L_1(\theta)\|^2 + \frac h2\| \nabla L_2(\theta)\|^2$. In this latter modified equation we see that the Lie bracket $[\nabla L_1, \nabla L_2]$ of the consecutive updates modifies the implicit flatness regularization on both updates. The Lie bracket defined as $[F, G] = (\nabla G)F - (\nabla F) G$ on general vector fields $F,G:\mathbb{R}^d\rightarrow \mathbb{R}^d$ is an important tool in differential geometry; see \cite{lee2012smooth_manifolds} for instance.

We approach the analysis of network growth similarly as a consecutive learning problem. In our new setting, the network before the growth operation is modelled by a small model loss while the model post-growth is modelled by a big model loss. We want to understand the additional inductive bias introduced by the growth step in the optimization. Therefore we are interested in computing the inductive bias generated by the two consecutive optimization steps surrounding the network growth operation.
%we first define two gradient vectors fields. 
The training of the small model and the large model follow respectively the gradient fields $F(w)=-\nabla_w L_1(w)$ and $G(w)=-\nabla_w L_2(w)$. Both vector fields are defined on the parameter space of the large model with $w_0=(\theta_1, \ldots, \theta_n, \eta=0)$ and subsequent training steps starting from $w=(\theta_{p_1}, \ldots, \theta_{p_n}, \eta)$. $\eta$ is the set of parameters in the large model that are not set with the small model and they are set to zero right after fusion.

%To approach the deep fusion method as a consecutive learning problem, we first define two gradient vectors fields. The training of the small model and the large model follow respectively the gradient fields $F(w)=-\nabla_w L_1(w)$ and $G(w)=-\nabla_w L_2(w)$. Both vector fields are defined on the parameter space of the large model with $w_0=\F(\theta_1, \ldots, \theta_n, \eta=0)$ and subsequent training steps starting from $w=\F(\theta_{p_1}, \ldots, \theta_{p_n}, \eta)$. $\eta$ is the set of parameters in the large model that are not set with the small model and they are set to zero right after fusion.

For $n$ models, immediately after fusion we form parameters $w_1$ and its gradient descent update is:
\begin{equation}
    w_1 = w_0 - h\nabla_w L_1(w_0) = w_0 + hF(w_0).
\end{equation}
Subsequent steps with vector field $G$ are defined as:
\begin{equation}
   \label{eq:w_2}
    w_2 = w_1 - h_B\nabla_w L_2(w_1) = w_1 + \alpha hG(w_1),
\end{equation}
where $h_B=\alpha h$ is the learning rate for the big model.

The losses $L_1(w)$ and $L_2(w)$ can be expressed in terms of the small and big models as follows:
\begin{align}
%   L_1(w) &= \DF(L_S(\theta_1), \ldots, L_S(\theta_n)) = \frac{1}{n} \left( L_S(\theta_1) + \ldots + L_S(\theta_n) \right), \\
  L_1(w) &= L \left( \frac{1}{n}\sum_{i=1}^{n} f_i(x;\theta_i), y \right), \\
  L_2(w) &= L_B(w) = L_B (\theta_{p_1}, \ldots, \theta_{p_n}, \eta), \label{eq:3}
\end{align}
where $x$ and $y$ are the inputs and reference labels, respectively, and $f_i(x; \theta_i)$ is the output for the $i$-th model with parameters $\theta_i$. Note that $L_1(w)$ is the loss right after fusion, so all small models are independent but the output is the average operator.
Also, for parameters of the form $w_0$ with $\eta=0$, we have $L_1(w_0)=L_2(w_0)$.

We present the true modified loss in Theorem~\ref{th:modified_loss} and the modified equation in Theorem~\ref{th:modified_equation}. These theorems suggest that in these network growth settings, the problem does not follow a gradient flow that is some weighted sum of the various losses (before and after the operation); rather, it is biased with the Lie bracket of the gradients' vector fields. More interestingly, the higher the big network's learning rate, the more prominent this bias is.

The modified loss in Theorem~\ref{th:modified_loss} corresponds with the implicit regularization definition from~\citep{igr}, taking into account the fusion operator.

When the Lie bracket between the two task gradients $[\nabla_w L_1, \nabla_w L_2]$
is not zero, this may actually interfere with this implicit regularization, potentially creating settings where the second update steers the learning trajectory away from the flatter regions of the first task.

\begin{restatable}[Modified loss]{theorem}{marginbound}
\label{th:modified_loss}
The consecutive gradient updates can be bounded by following the gradient flow on this modified loss:
\begin{align}
    \tilde{L}(w)
    & \leq \frac{1}{n}\sum_{i=1}^{n}L_S(\theta_i) + L_B(w) + \frac{h}{4n^2}\sum_{i=1}^{n} \lVert \nabla_{\theta_i} L_S(\theta_i) \rVert^2 \nonumber \\ &+ \frac{h\alpha^2}{4} \lVert \nabla_w L_B(w) \rVert^2.
\end{align}
where $L_S(\theta_i)$ is the loss for the $i$-th small model.

\end{restatable}

For the special case of self-fusion (i.e., the same model is fused with itself), the actual modified loss is,
\begin{align*}
    \tilde{L}(w)
    &= L_S (\theta) + L_B(w) + 
     \frac{h}{4} \lVert \nabla_\theta L_S(\theta) \rVert^2 \nonumber \\ &+ \frac{h\alpha^2}{4} \lVert \nabla_w L_B(w) \rVert^2.
\end{align*}

\begin{restatable}[Modified equation]{theorem}{modifiedEquation}
\label{th:modified_equation}
Consider two consecutive training runs like the ones described above, first a small model, then deep fusion and finally a large model training. The modified equation shadowing the update composition is of the form:
\begin{equation}
    \dot{w}(t) = -\nabla_w \Tilde{L}(w_t) + \frac{h\alpha}{2} [\nabla_w L_1,\nabla_w L_2](w_t) + {\cal O}(h^2). \nonumber
\end{equation}
\end{restatable}

%From Theorem~\ref{theorem:expanded_loss} we can extract the following ideas:
%
% \begin{itemize}
%     \item The larger the learning rate in the big model the more influence of the Lie bracket and the faster we transport the small space of parameters to the big model;
%     \item Because the norm of the gradient of the small model is divided by a quadratic n, the more models the smaller that term, hence the smaller the modified loss;
%     \item We need to balance alpha to make the Lie bracket more prominent but without making the norm of the gradient of the larger model too big which would make the modified loss also too big. I think this is a nice outcome of the current definition and maybe we can extract some proportions for n, alpha and number of parameters.
% \end{itemize}
\begin{corollary}
The larger the learning rate in the big model the more influence of the Lie bracket. % and the faster we transport the small space of parameters to the big model
\end{corollary}

In the coming section, we analyze this bias and understand its structure. Later in the experiment section, we empirically measure its importance.

%%%%%%%%%%%%%%%%%%%%%%%%%%%%%%%%%%%%%%%%%%%%%%%%%%%%%%%%%%%%%%%%%%%%%%%%%%%%%%%%%%%%%%%%%%%%%%%%

\subsection{Lie Bracket Structure vs the Gradient Loss}
In the preceding section, we observed that the Lie bracket serves as a crucial adjustment factor to the gradient difference between the loss of the small model and the fused model. This section examines the Lie bracket from two perspectives. Initially, Lemma~\ref{lemma:same_gradient} confirms that the updated model parameters $\eta$ follow the same gradient updates as the original small model parameters. The Lie bracket's essential role in fused models becomes evident here; without it, the model updates would be uniform, lacking any additional effects. Subsequently, Lemma~\ref{lemma:non_zero} illustrates that the Lie bracket ensures the updated parameter directions are the only non-zero components, highlighting the Lie bracket's significance in refining our algorithm.

For the sake of the analysis, the following discussion assumes a self fusion environment of $n$ models. In the self fusion case, the output of the big model right after fusion is $\hat{y} = \frac{1}{n} \sum_{i=0}^n f_i(x, \theta_i) = f(x, \theta)$ given that all models are identical.

\begin{restatable}[Scaled gradient]{lemma}{scaledGradient}
\label{lemma:scaled_gradient}
%At expansion, the gradient from the loss is equivalent for all four blocks of the fusion operation.
Let $f$ be a small model with parameters $\theta$. Suppose its loss $L_S$ is continuous and that $L_S$'s partial derivatives are Lipschitz continuous. When performing self deep fusion $n$ times of $f$, for the first gradient step after fusion, $\nabla_{\theta} L_B(\theta, \ldots, \theta, \eta=0)$, we have 
$$
\nabla_{\theta} L_B(\theta, \ldots, \theta, \eta=0) = \frac 1n \nabla_{\theta} L_S(\theta).
$$
\end{restatable}

\begin{restatable}[Same gradient]{lemma}{sameGradient}
\label{lemma:same_gradient}
%At expansion, the gradient from the loss is equivalent for all four blocks of the fusion operation.
Let $f$ be a small model with parameters $\theta$. When performing self deep fusion $n$ times of $f$, for the first gradient step after fusion for every $\eta_i$, we have that 
$$
\nabla_{\eta_i} L_B(\theta, \ldots, \theta, \eta=0) = \frac 1n \nabla_{\theta_j} L_S(\theta),
$$
for some $\theta_j$ that is a parameter in the same layer.
\end{restatable}
 
Lemma~\ref{lemma:scaled_gradient} and Lemma~\ref{lemma:same_gradient} suggests that the new parameters~$\eta$ are updated using a gradient analogous to the combined gradients from the small models. This means that the $\eta$ parameters are updated by concatenating the $\theta_{p_{i}}$ parameters from the small model. Without an implicit Lie bracket term in subsequent gradient steps, the integration of the $\eta$ parameters would be unfeasible.

\begin{restatable}[Non-zero Lie bracket]{lemma}{nonZeroLieBracket}
\label{lemma:non_zero}
In self deep fusion, the $[\nabla_w L_1,\nabla_w L_2]$ has non-zero values only for the parameter dimension $\eta$ that was initialized as zero in the big model.
\begin{equation*}
[F, G] =
\begin{pmatrix}
0 \\
\vdots \\
0 \\
 (\sum_{i=1}^n \nabla_{\theta_i}\nabla_\eta L_B(w))\nabla L_S(\theta)
\end{pmatrix}
\end{equation*}

% \begin{equation*}
% [F, G] = \begin{pmatrix}
% 0 \\
% \vdots \\
% 0 \\
% - n(\nabla_{\theta}\nabla_\eta L_B(w)\nabla L_S(\theta)
% \end{pmatrix}.
% \end{equation*}

\end{restatable}

In Lemma~\ref{lemma:non_zero}, the analysis requires viewing the problem through the lens of manifolds. Consider $M$ as the manifold of the big model, defined as $M=\{ (\theta, \ldots, \theta, \eta): \theta \in \mathbb{R}^{d_s}, \eta \in \mathbb{R}^{(d_b - d_s)} \}$. The fused model (right after fusion) resides within a smaller submanifold $N \subset M$, where $N = \{ (\theta, \ldots, \theta, 0): \theta \in \mathbb{R}^{d_s} \}$. This lemma demonstrates that the Lie bracket acts as a mathematical tool for understanding the transfer of information between the submanifold $N$ and the ambient manifold $M$.

The Lie bracket $[ F, G ]$ serves as a correction mechanism that integrates the influence of the submanifold $N$ into the overall system dynamics of $G$.
Essentially, the Lie bracket facilitates the exchange of directions between $N$ and $M$.

\begin{corollary}\label{corollary:small_grad_mag}
If for the small model $\nabla L_S(\theta)$ is close to zero, then the Lie bracket on $N$ is close to zero.
\end{corollary}

%%%%%%%%%%%%%%%%%%%%%%%%%%%%%%%%%%%%%%%%%%%%%%%%%%%%%%%%%%%%%%%%%%%%%%%%%%%%%%%%%%%%%%%%%%%%%%%%%%%%

\section{Experiments}\label{experiments}
%In this section we present experiments with the fusion technique. 
We begin by training T5 language models on the C4 dataset. The term `deep' will be dropped when context allows.

\subsection{Language Models}
The aim of this experiment is to establish a better initial checkpoint for a large T5 \cite{t5paper} transformer network, referred to as \textsc{T5-Medium}, by using two smaller T5 models, denoted as \textsc{T5-Small}. We present two types of results: fusing two unique small models and fusing one model with itself (self fusion).
% \begin{table}[h]
% \centering
% \begin{footnotesize}
% \begin{tabular}{c|c|c}
% Model Name & T5-Small & T5-Medium \\
% \hline
% embedding dim & 512 & 1024 \\
% number of heads & 6 & 12 \\
% enc./dec. layers & 8 & 8 \\
% head dim & 64 & 64 \\
% mlp dimension & 1024 & 2048 \\
% \hline
% number of parameters & 77M & 242M 
% \end{tabular}
% \end{footnotesize}
% \caption{Dimensions of T5 Small and Medium.}
% \label{tab:models_dim1}
% \end{table}
We trained the following 4 experiments (see dimensionalities in Table \ref{tab:models_dim1} in Appendix~\ref{apndx:first_experiment}):
\begin{enumerate}
    \item \texttt{baseline}: \textsc{T5-Medium} from random initialization.
    \item \texttt{fusion-rule}: \textsc{T5-Medium} trained from fusing the two \textsc{T5-Small} models each trained for 1M steps, while maintaining the fusion rule.
    \item \texttt{fusion-prop}: \textsc{T5-Medium} trained from fusing the two \textsc{T5-Small} models each trained for 1M steps, while maintaining the fusion property.
    \item \texttt{self-fusion}: \textsc{T5-Medium} trained from self fusing a \textsc{T5-Small} model trained for 1M steps.
\end{enumerate}

%All above experiments ran on Google TPU v3 of topology~4x4. Additionally, 
Zero matrices in Eq.~\ref{eq:diagonal} were substituted with blocks initialized randomly with a low variance. Table~\ref{performance1} presents performance comparison between the various fusion algorithms and their cost.
%Final results are displayed in Table~\ref{performance1}, and Figure \ref{fig:accuracy_and_loss} shows the evaluation metric curves throughout the training.

\begin{table}[ht]
\centering
\begin{scriptsize}
\begin{tabular}{cccc}
    \footnotesize{\textbf{Model}} & \footnotesize{\textbf{Loss}} & \footnotesize{\textbf{Accuracy}} & \footnotesize{\textbf{Cost}} \\
    \toprule
    % \texttt{baseline} & 4.66e+4 & 66.65 $\pm$ 0.01 \\
    \texttt{fusion-rule} & 4.61e+4 & 66.88 & 2 $\cdot$ 16h + 37.4h = 53.4h \\
    \texttt{fusion-prop} & \textbf{4.53e+4} & \textbf{67.25} $\pm$ 0.03 & 2 $\cdot$ 16h + 41h = 73h \\
    \texttt{self-fusion} & 4.55e+4 & 67.20 $\pm$ 0.05 & \textbf{16h + 42.4h = 58.4h}  \\\bottomrule
\end{tabular}
\end{scriptsize}
\caption{Performance of different T5-Medium fusion methods at 1 million steps, replicated three times for standard deviation. Cost is in TPU V3 4x4 hours.}
\label{performance1}
\end{table}
% Baseline: https://tensorboard.corp.google.com/compare/double-size-3:5477766489616969273,double-size-2:2566067252798318755,pretrain:2514334825984166543/?runFilter=eval#timeseries
% Fusion: https://tensorboard.corp.google.com/compare/sanity-check:1682347774996939924,2-1m:2081380001566878025,3-1m:6154192545999356339/?runFilter=eval#timeseries
% Self fusion: https://tensorboard.corp.google.com/compare/pretrain:1405147768129656998,self-fusion-3:7644870703437819787,self-fusion-exp1:7202102333169765384/?runFilter=eval#timeseries

\begin{figure}[ht]
\centering
\includegraphics[scale=0.18]{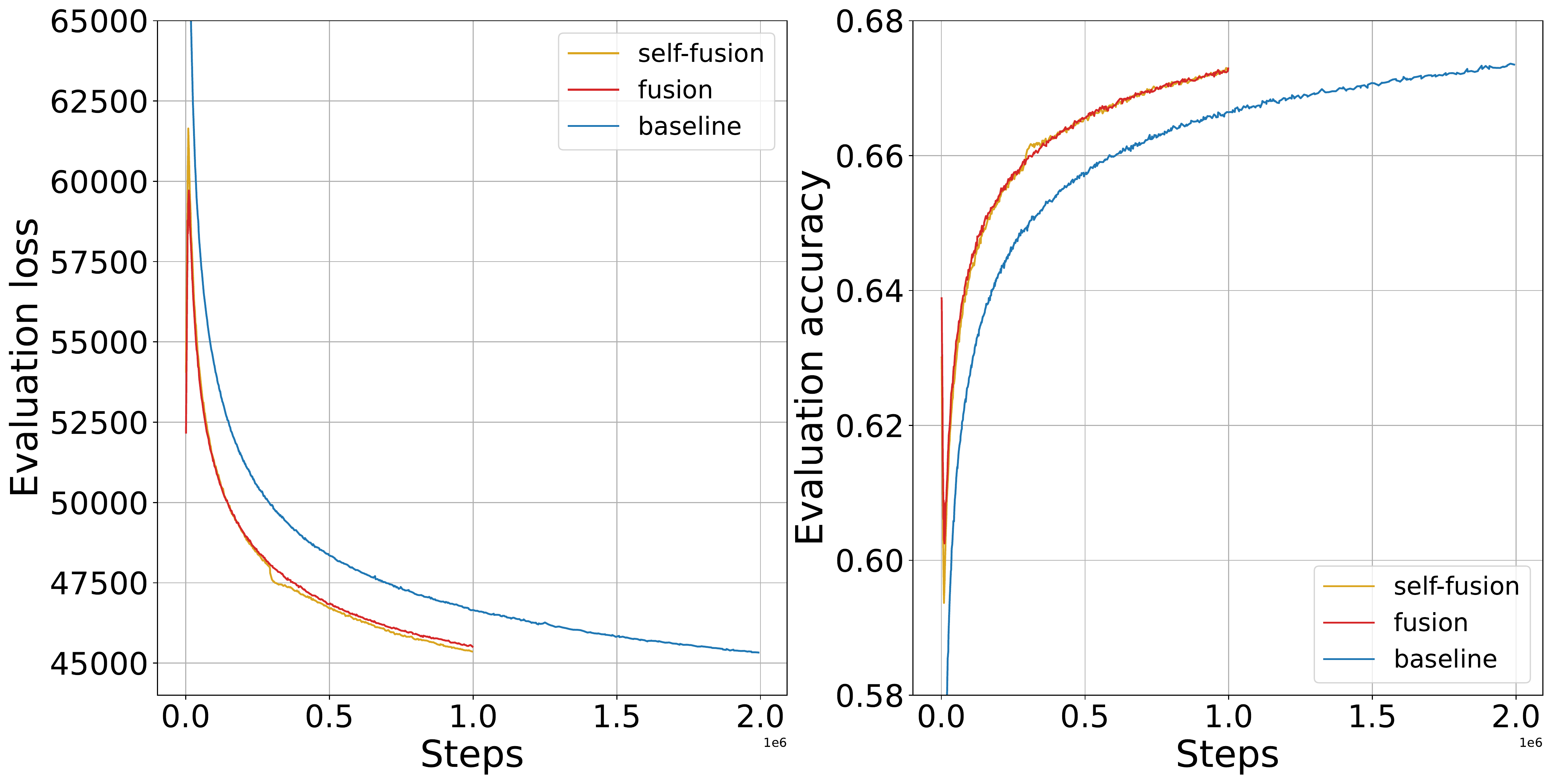}
\caption{Accuracy and loss on validation data. The x-axis of the graph is scaled in millions of steps.} %of fusion}% (red) and self fusion (yellow) \label{accuracy_and_loss}
\label{fig:accuracy_and_loss}
\end{figure}

The outcomes of our experiments indicate that while it requires extra code changes to the T5 transformer, upholding the fusion property results in superior performance compared to adhering to the fusion rule. Furthermore, we discovered that self fusion yields comparable performance to standard fusion. 

As for comparing against the baseline, Figure~\ref{fig:accuracy_and_loss} shows the learning curves of the top two out of three fusion algorithm compared to baseline. We discovered that the baseline needed
%Significantly, the \texttt{baseline} required 
an additional 860K steps to achieve the performance level of self fusion. When employing self fusion, training a \textsc{T5-Medium} {\it resulted in an 18\% reduction in computation time} compared to the \texttt{baseline}. \footnote{\textsc{T5-Small} model training time included.} Details are in Table~\ref{hour_performance_for_5_1}.

\begin{table}[h]
\centering
\begin{footnotesize}
\begin{tabular}{ccccc}
\textbf{Model / time} & \textbf{Fusion} & \textbf{Post fusion} & \textbf{Cost} & \textbf{Acc.} \\
%& & & time & avg \\
\toprule
    \texttt{baseline} & 0 steps & 1.86M steps &71.2h & 67.2\\
     & 0h & 71.2h & & \\ \midrule
    \texttt{self-fusion} & 1M steps & 1M steps & \textbf{58.4h} &  67.2\\
     & 16h & 41h &  & \\\bottomrule
\end{tabular}
\end{footnotesize}
\caption{Cost is in TPU hours (TPU V3 4x4 topology).}
\label{hour_performance_for_5_1}
\end{table}

\subsection{Fusion in Stages}
We explored staged fusion using T5-S, T5-M, and T5-L architectures (Table \ref{models_dim2}, Appendix~\ref{apndx:second_experiment}) and tested various fusion settings depicted in Figure~\ref{topology}.

\begin{figure}[ht]
\centering
\begin{adjustbox}{max width=\linewidth}
\input{images/settings}
\end{adjustbox}
\caption{Settings for final T5-L fusion: yellow signifies fused models, white indicates regular training, and links represent fusion (double link signifies self fusion). Every node in the graph is trained 1M steps (as an example - algorithm 4 is trained a total of~7M steps).}
\label{topology}
\end{figure}

Every model (T5-S, T5-M, T5-L) is trained 1M steps.  Table~\ref{topology_table} below present the performance of the various models, temporarily disregarding the cost.

\begin{table}[h]
\centering
\begin{footnotesize}
\begin{tabular}{ccc}
\textbf{Model} & \textbf{Loss} & \textbf{Accuracy} \\
\toprule
(1) & 4.04e+4 & 69.89 \\
(2) & 3.93e+4 & 70.45 \\
(3) & 3.9e+4 & 70.56 \\
(4) & 3.87e+4 & 70.74 \\
(5) & 3.91e+4 & 70.57 \\
(6) & 3.91e+4 & 70.47 \\\bottomrule
\end{tabular}
\end{footnotesize}
\caption{Performance of the various ways of fusing T5-L.}
\label{topology_table}
\end{table}

The results show similar performance between fusion and self fusion (settings (3) and (5)) as we seen in the previous experiment. However, repeated self fusion reduces performance, while multiple regular fusions enhance T5-L performance (settings (6) vs (4)). While repeated regular fusion enhance performance, it is very costly, and thus the best performance if we take cost into consideration is self fusion again. Training a model using a single application of self fusion, setting (5), results in a {\it 20\% reduction in computation time on T5-L} compared to the standard setting~(1). Details in Table~\ref{hour_performance_for_5_2} below.

\begin{table}[h]
\centering
\begin{footnotesize}
\begin{tabular}{ccccc}
\textbf{Model / time} & \textbf{Fusion} & \textbf{Post fusion} & \textbf{Cost} & \textbf{Acc.} \\
%& & & time & avg \\
\toprule
    (1) & 0 steps & 1.7M steps & 246.7h & 70.57\\
     & 0h & 246.7h & & \\ \midrule
    (5) & 1M steps & 1M steps & \textbf{197.9h} &  70.57 \\
     & 47.7h & 150.2h &  & \\\bottomrule
\end{tabular}
\end{footnotesize}
\caption{Cost is in TPU hours (TPU V3 4x4 topology). Baseline (setting (1)) needed 700K extra steps to reach performance of self fusion (setting (5)).}
\label{hour_performance_for_5_2}
\end{table}

%The applied once self fusion (model 5), trains to its performance {\it with 20\% less compute time} than the vanilla training baseline. 

\subsection{Fine Tuning for Down Stream Tasks} \label{downstream}
%To further understand the models trained with fusion, we fine tuned some checkpoints to understand performance on various NLP tasks (using the GLUE benchmark).

We fine-tuned high performing settings from the first experiment together with a baseline on NLP tasks using the GLUE benchmark.
%Armed with the knowledge from previous subsections, we fine-tuned the higher performance settings in addition to the baseline from the first experiment.  
We trained two \textsc{T5-Small} models for 500K steps before self fusing them to create a \textsc{T5-Medium}. We also trained a standalone \textsc{T5-Medium} to show the difference with a randomly initialized model of the same size. 
These models were pretrained for ~0 (corresponding to \texttt{baseline} without pretraining vs fusion without extra pretraining), 250K, 500K, and 1M steps (\texttt{baseline} only), and then fine-tuned. The GLUE average results are shown in Table \ref{finetune_table} and Figure \ref{glue_fig}. The complete results for each task is presented Table~\ref{tbl:full_glue} in the Appendix.

\setlength{\tabcolsep}{5pt}
\begin{table}[h]
\centering
\begin{footnotesize}
\begin{tabular}{ccccc}
    \textbf{\tiny{Model / Pretrain steps}} & 0 & 250K & 500K & 1M\\
    \toprule
    \texttt{baseline} & 64.07 & 83.33 & 84.35 & 84.74$\pm$0.13 \\
    \texttt{fusion-prop} & \textbf{81.40} & \textbf{84.10} & 84.86$\pm$0.13 & - \\
    \texttt{self-fusion} & 81.01 & 83.71 & \textbf{84.94}$\pm$0.2 & - \\\midrule
    \textsc{T5-Small} & - & - & - & 80.28 \\\bottomrule
\end{tabular}
\end{footnotesize}
\caption{Performance (GLUE average) of the various models on downstream tasks, replicated three times for standard deviation.}
\label{finetune_table}
\end{table}

\begin{figure}[t]
\centering
\includegraphics[scale=0.18]{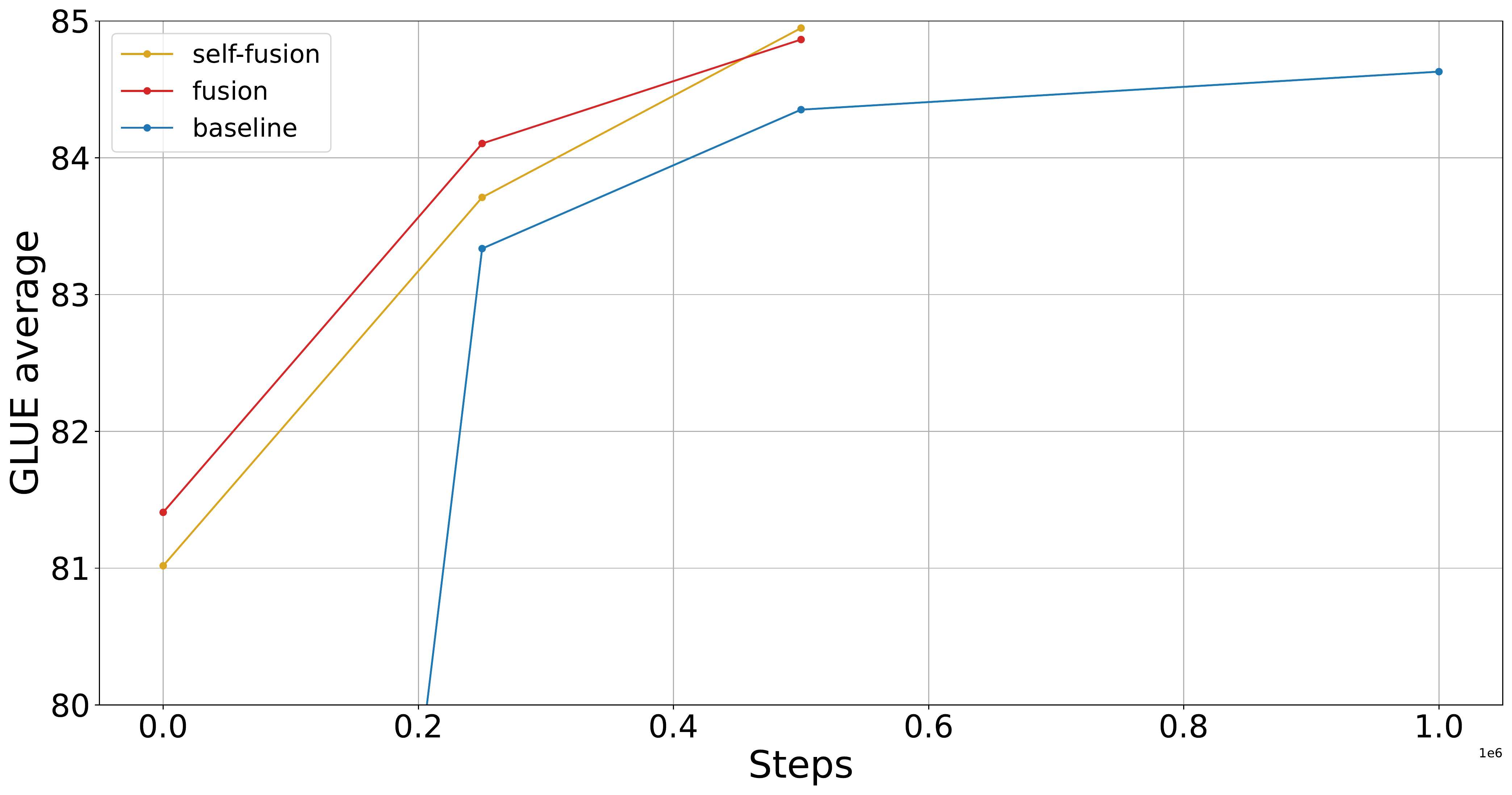}
\caption{Performance (Glue average - an average over many NLP tasks that score between 0 and 100) of the various models.}
\label{glue_fig}
\end{figure}

Our results indicate that enhancing a pretrained model's performance may simply require self-fusion before fine-tuning, without further pretraining. For instance, a \textsc{T5-Small} model, trained for 500K steps, when self-fused and fine-tuned, outperforms the small model trained to 1M steps before fine-tuning (81.01 vs 80.28). It is evident that the extra parameters from self-fusion benefit NLP tasks more than extended pretraining.

Next, the results above also suggest that deep fusion can lead to faster training to better performance, when fine-tuning on downstream NLP tasks. However, while in pretrain, the training curves of fusion and self fusion look similar, we can see that for downstream tasks, fusion maintains higher performance throughout until convergence (here, both models converge to similar performance).
\begin{table}[h]
\centering
\begin{footnotesize}
\begin{tabular}{ccccc}
\textbf{Model / time} & \textbf{Fusion} & \textbf{Post fusion} & \textbf{Time} & \textbf{GLUE} \\
%& & & time & avg \\
\toprule
    \texttt{baseline} & 0 steps & 1M steps &39.2h & 84.74\\
     & 0h & 39.2h & & \\ \midrule
    \texttt{fusion-prop} & 500k steps & 500k steps & 37.9h &84.86 \\
     & 2$\times$8h & 21.9h & & \\ \midrule
    \texttt{self-fusion} & 500k steps & 500k steps & \textbf{29.9h} &  \textbf{84.94}\\
     & 8h & 21.9h &  & \\\bottomrule
\end{tabular}
\end{footnotesize}
\caption{Compute time in hours (TPU V3 4x4 topology).}
\label{hour_performance}
\end{table}

The total compute saving is about 24\% TPU time for this configuration as presented in Table \ref{hour_performance}. Even though we trained for less time, the final performance was slightly better than the baseline.
\subsection{Training Dynamics}
In this section, our experiments will show how the post-fusion learning rate affects the performance of the learning, as well as the parameters. To understand how the learning rate affects performance, we ran the normal T5 learning rate schedule with various offsets. A positive offset means we start from lower maximal learning rate while a negative offset means we maintain the maximal learning rate for longer before dropping. Figure~\ref{lr_sparse_fig} displays the performance at the extremes (high positive offset vs high negative one). The full tested spectrum of offsets tested appears in Figure~\ref{lr_all_fig} in Appendix~\ref{apndx:train_dynamics}.

% Code in: https://colab.corp.google.com/drive/1V6_Hp1cG_hcvKW86vkzQqIxcXV5bfvUv?resourcekey=0-NxC1wsba8TOwi_5upCnvRQ#scrollTo=CkYSBC9y6CEN
\begin{figure*}[t]
\centering
\includegraphics[scale=0.20]{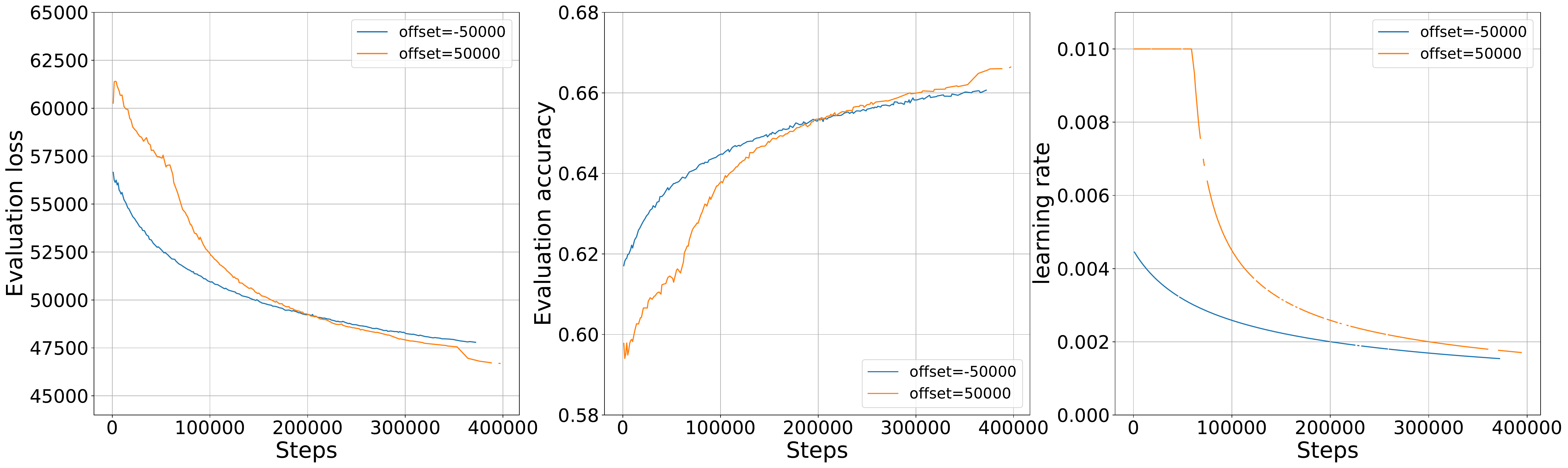}
\caption{Performance on T5 when applying large offsets to the learning rate schedule.}
\label{lr_sparse_fig}
\end{figure*}
% https://colab.corp.google.com/drive/1WaxW8rkC-Xb0v4EFxwm-zWzOJvGzUj2A?resourcekey=0-4WVaVs6jYXPMBVH-2EKjBA#scrollTo=Vlt57U5-QfVt
\begin{figure*}[t]
\centering
\begin{tabular}{c}
\begin{minipage}{.125\linewidth}
    \includegraphics[width=\linewidth]{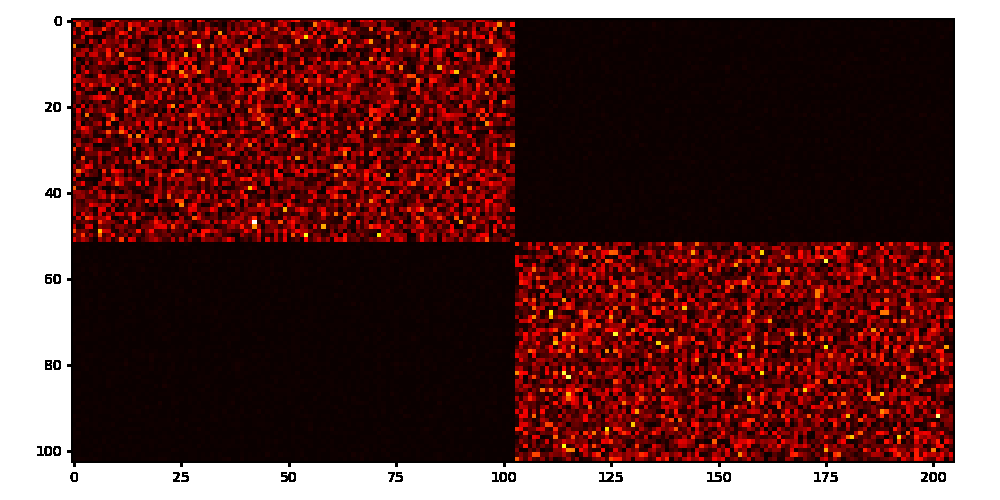}
\end{minipage}%
\begin{minipage}{.125\linewidth}
    \includegraphics[width=\linewidth]{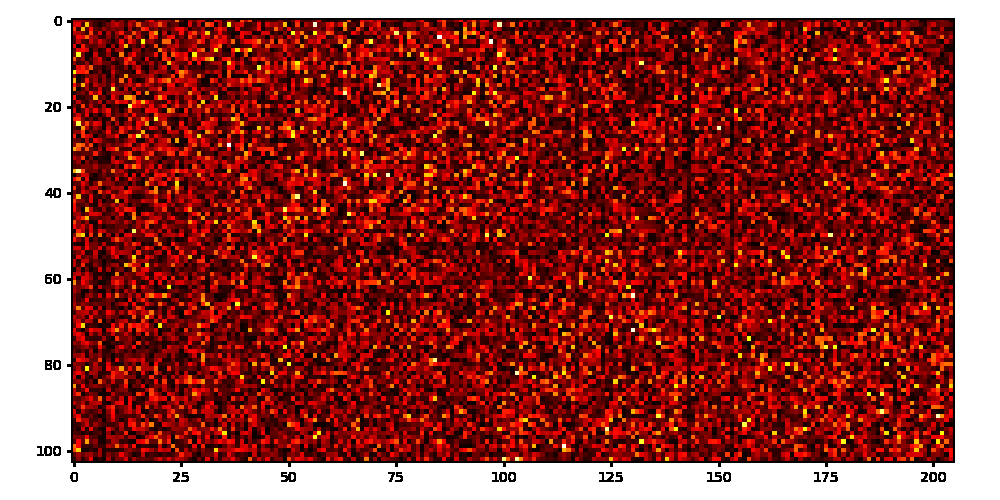}
\end{minipage}%
\begin{minipage}{.125\linewidth}
    \includegraphics[width=\linewidth]{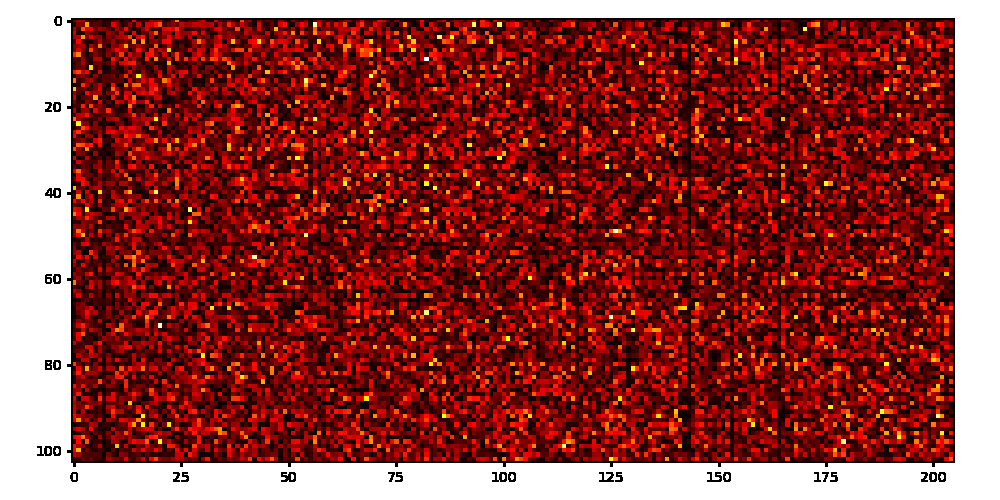}
\end{minipage}%
\begin{minipage}{.125\linewidth}
    \includegraphics[width=\linewidth]{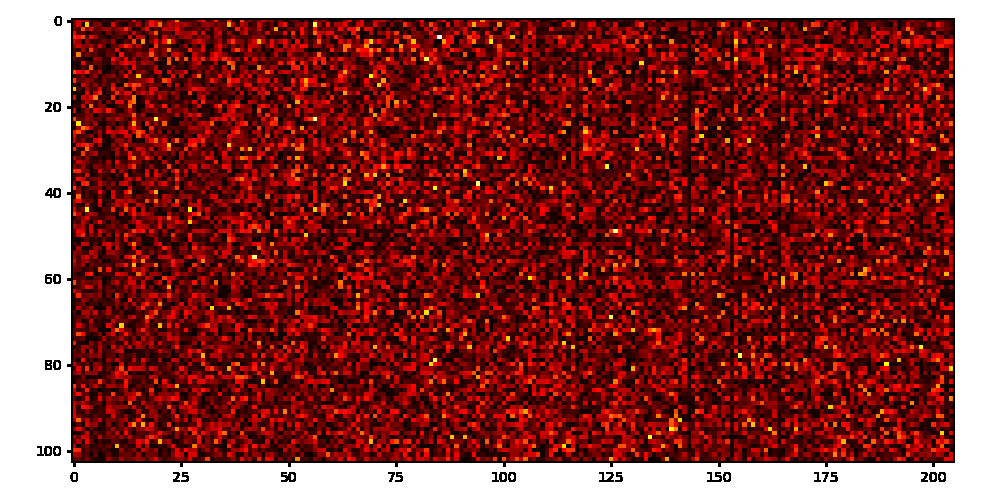}
\end{minipage}%
\begin{minipage}{.125\linewidth}
    \includegraphics[width=\linewidth]{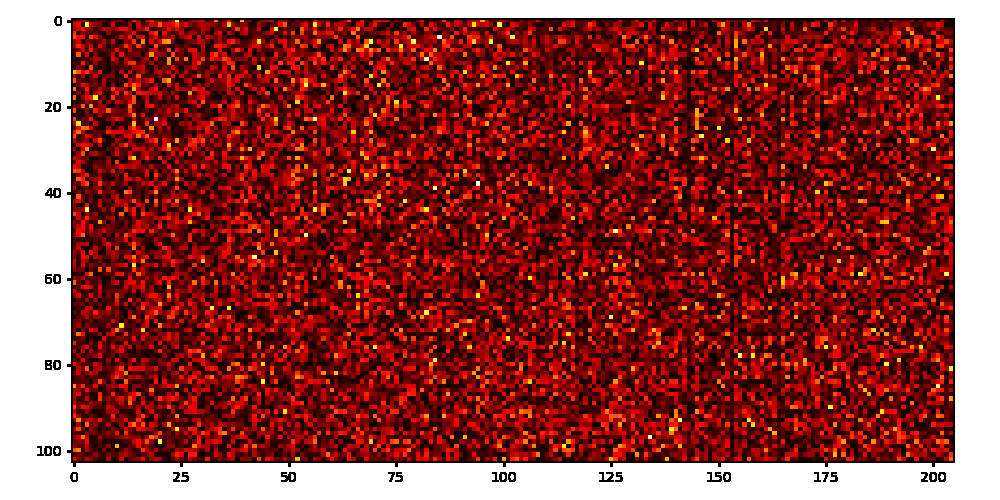}
\end{minipage}%
\begin{minipage}{.125\linewidth}
    \includegraphics[width=\linewidth]{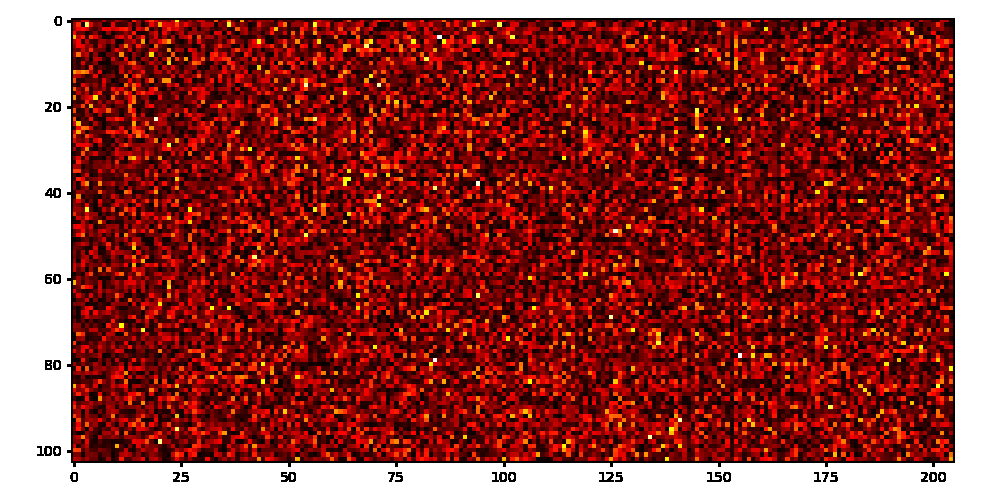}
\end{minipage}%
\begin{minipage}{.125\linewidth}
    \includegraphics[width=\linewidth]{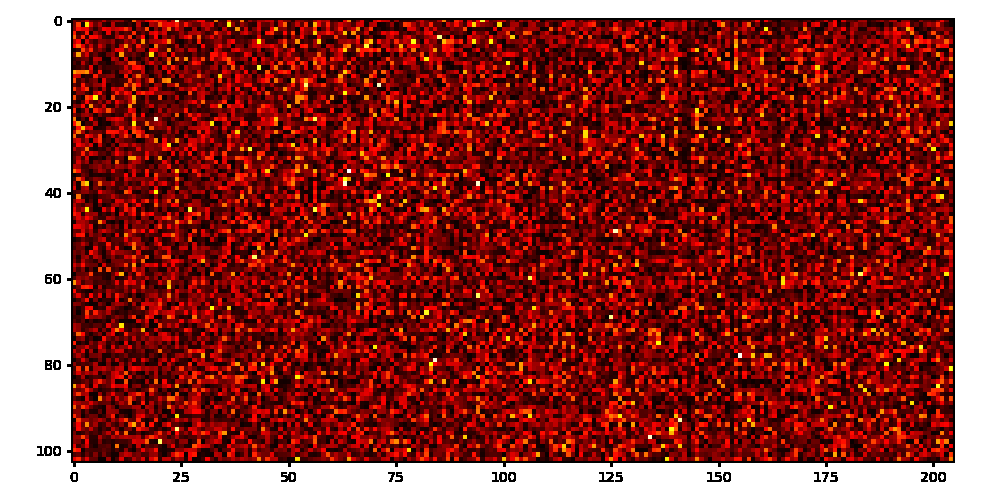}
\end{minipage}%
\begin{minipage}{.125\linewidth}
    \includegraphics[width=\linewidth]{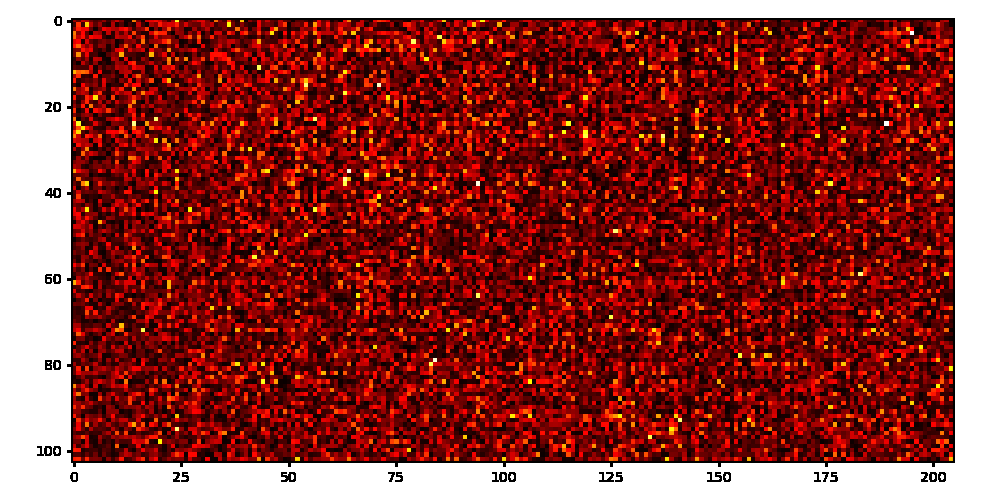}
\end{minipage} \\
\begin{minipage}{.125\linewidth}
    \includegraphics[width=\linewidth]{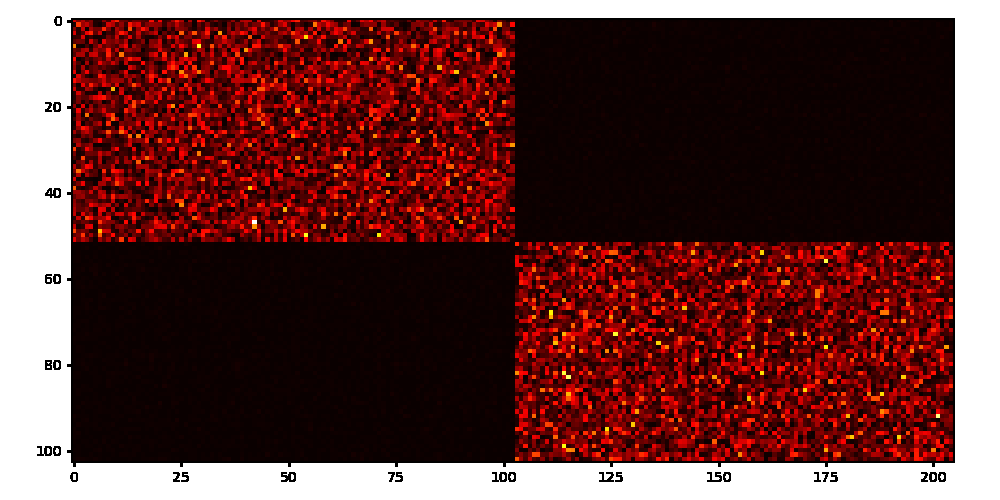}
\end{minipage}%
\begin{minipage}{.125\linewidth}
    \includegraphics[width=\linewidth]{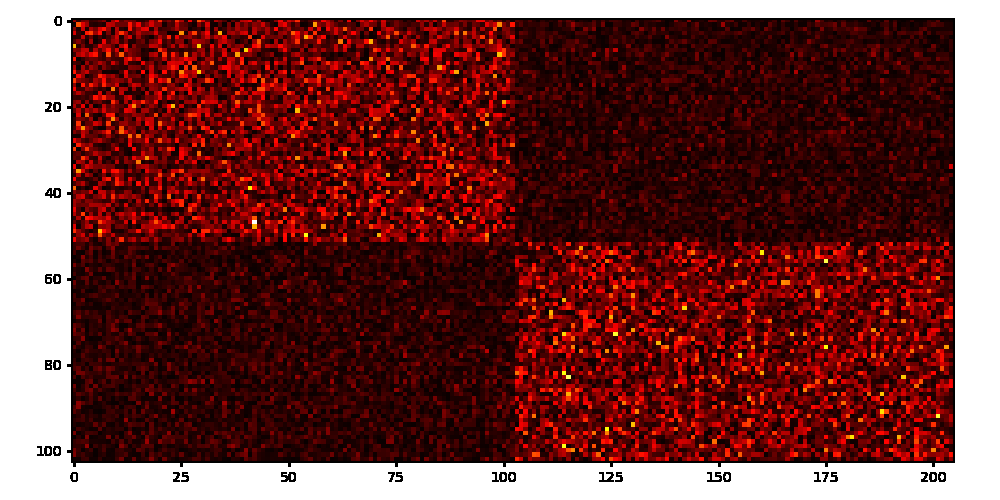}
\end{minipage}%
\begin{minipage}{.125\linewidth}
    \includegraphics[width=\linewidth]{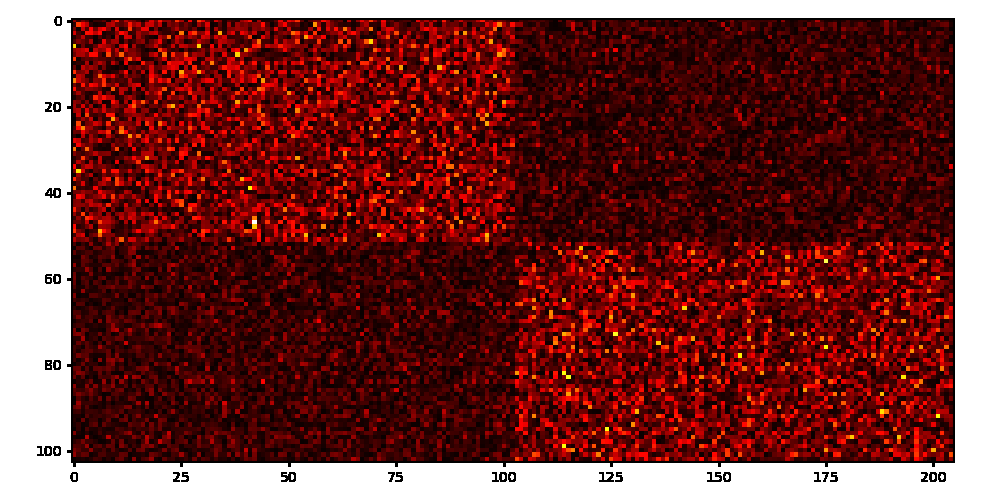}
\end{minipage}%
\begin{minipage}{.125\linewidth}
    \includegraphics[width=\linewidth]{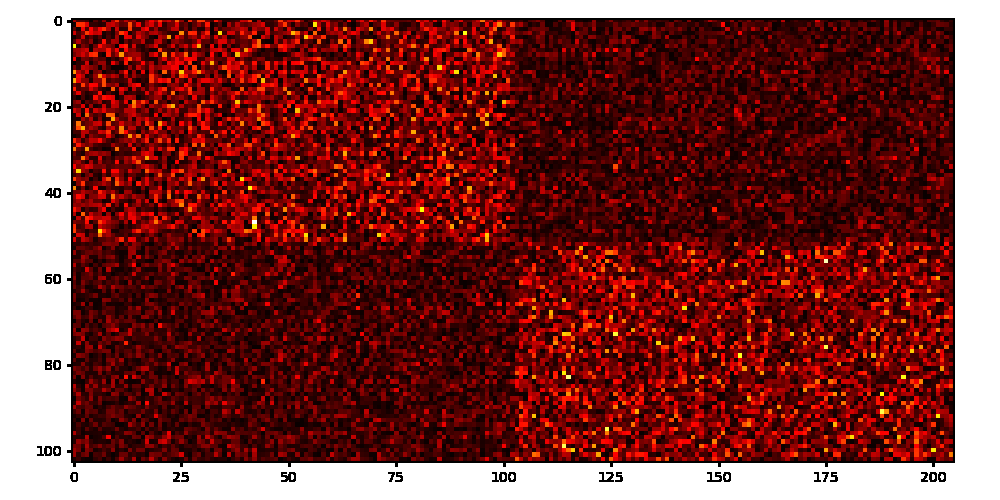}
\end{minipage}%
\begin{minipage}{.125\linewidth}
    \includegraphics[width=\linewidth]{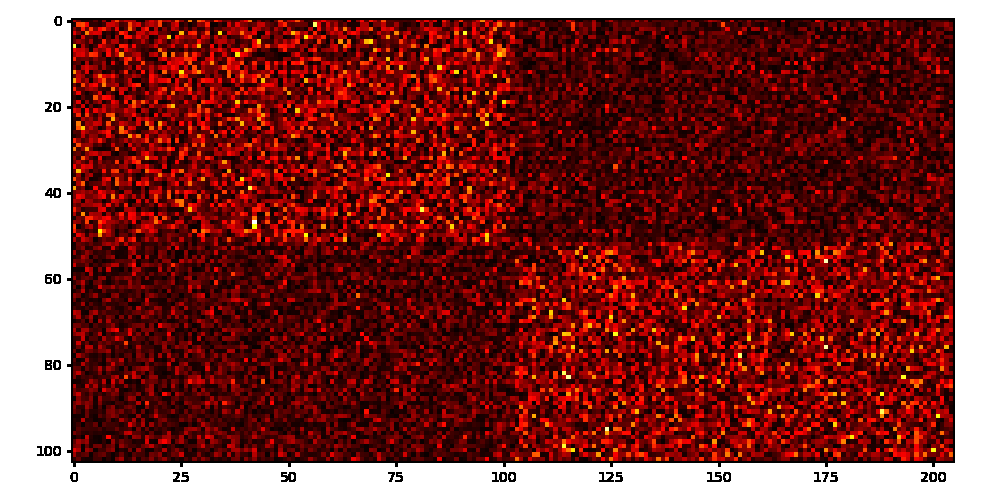}
\end{minipage}%
\begin{minipage}{.125\linewidth}
    \includegraphics[width=\linewidth]{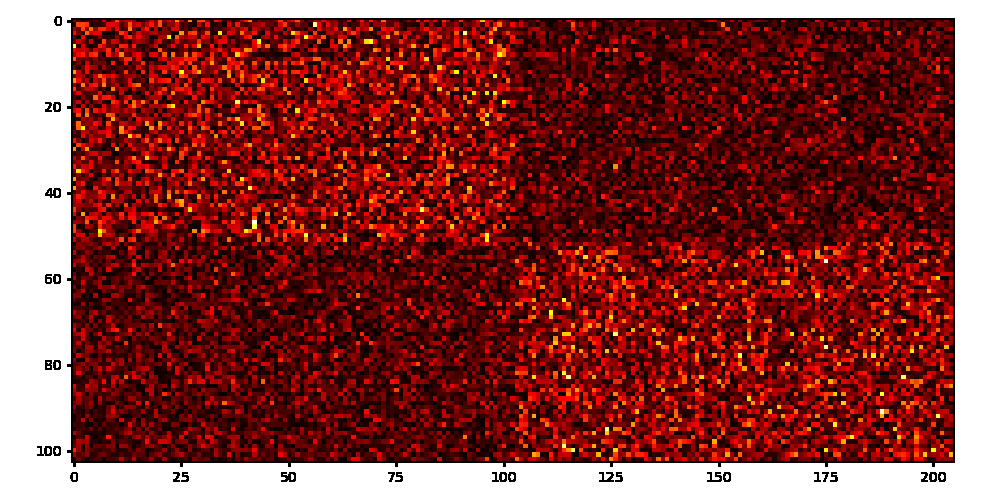}
\end{minipage}%
\begin{minipage}{.125\linewidth}
    \includegraphics[width=\linewidth]{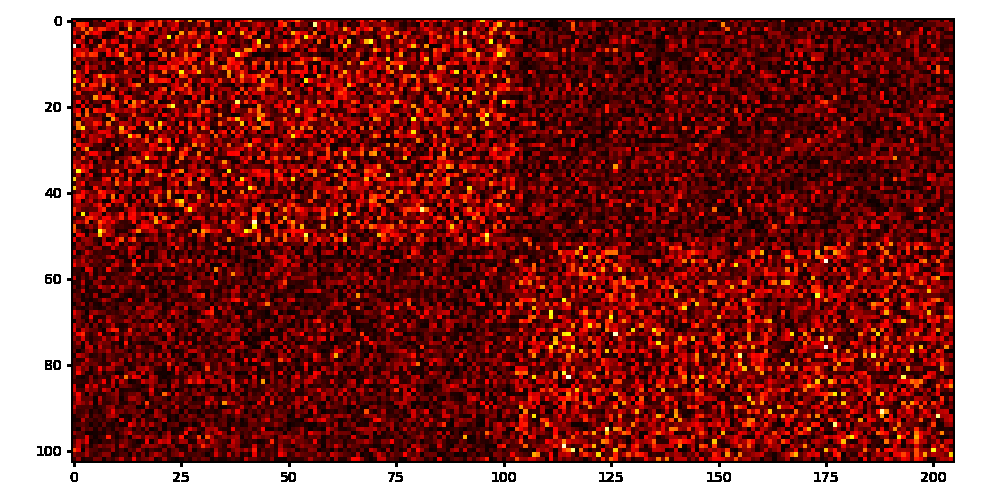}
\end{minipage}%
\begin{minipage}{.125\linewidth}
    \includegraphics[width=\linewidth]{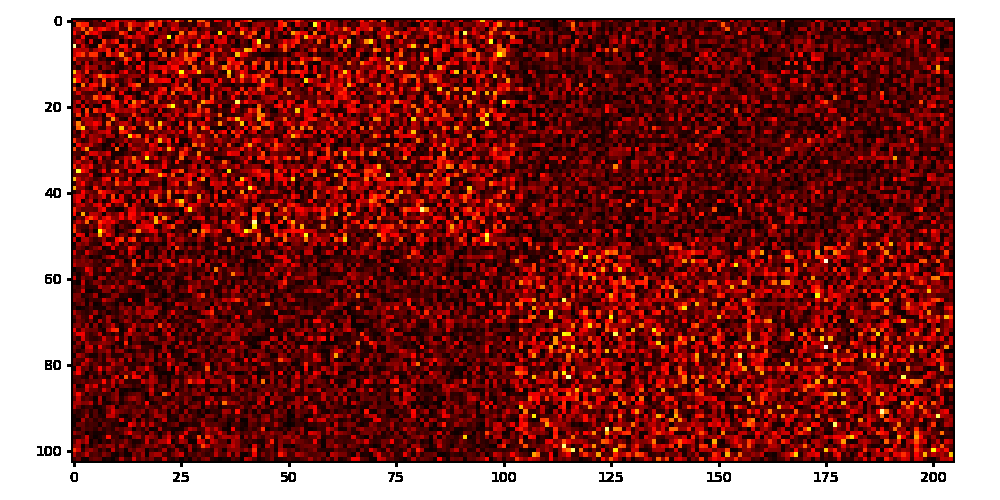}
\end{minipage}%
\end{tabular}
\caption{Heat maps of the first layers' feed forward kernel every 50K steps (left to right till 400K steps training). The upper row is for high learning rates (highly negative offset: -50K), while the lower row is displaying the heat maps for low learning rates (highly positive offset: +50K). Figure~\ref{fig:zoomin_heatmap} in the Appendix zooms in on the last heat map in the first row showcasing how the replicas of the smaller model diverge.}
\label{fig:heatmaps}
\end{figure*}

While it is clear that higher learning rates toward the beginning benefit learning in the long run, we observe they hurt performance at first. This is in line with Section~\ref{theory} where we show that larger Lie bracket values bias the updates to unlock capacity, but at the short term cost of higher objective loss. Adding a negative offset resulted in better performance, but led to instabilities, and thus the previous experiment did not introduce offsets to the learning rate schedule.

To further understand the drop in performance for low learning rates, we use heat maps to visualize the kernels from the first feed forward layer. Figure~\ref{fig:heatmaps} shows the kernel when the learning rate is high and low. It is evident that after 400K steps, the parameters initialized to zero were not able to catch up in magnitude with the already optimized ones at expansion point. This observation is further aligned with Section~\ref{theory} where we note that high learning rates lead to high Lie bracket values, thereby unlocking the potential of the extra parameters.

\subsection{When to Perform Fusion}
% Lastly, we ran extensive experiments to analyze the optimal \deletetext{time} \newtext{step} to fuse. 
% \deletetext{In particular, we trained a \textsc{T5-SMALL} model, self fused at multiple points throughout the model training, and continued training each self-fused model after fusion. We then analyzed model quality after a given time budget.}
Suppose we are given $X$ hours budget of TPU/GPU time, and suppose we would like to grow the network once through training (via self deep fusion), we wanted to answer the question: at which step is it optimal to perform self deep fusion?

We performed experiments training \textsc{T5-Medium} (Table~\ref{tab:models_dim1} in Appendix~\ref{apndx:first_experiment}) with various budgets: 50, 60 and 70 hours. The final model was obtained by first training \textsc{T5-Small} and apply self deep fusion at step $x$ during training for various possible steps $x$ (see Figure~\ref{fig:time_to_fuse}). In the graph, each point represents a different training, where the X-axis represents the step (in 1000s) in which we applied self deep fusion, and the Y-axis represents the final accuracy of the model.

We found that selecting the optimal time to fuse for a given time budget requires some balance.

In particular: 
\begin{itemize}
    \item If the small model has not been trained enough the final fused model will under-perform.
    \item If the small model is trained and has fully converged, the small model is allotted too much of the budget and the final model shows a slight degradation from optimal. 
\end{itemize} 

For the \textsc{T5-SMALL} model, fusing at ~500k steps performs well (Figure~\ref{fig:time_to_fuse}). We note in our experiments the optimal number of steps for the small model was independent of the total budget -- a confirmation of the importance of the information transfer when fusing from small to the larger models.

\begin{figure}[h]
\includegraphics[width=\linewidth]{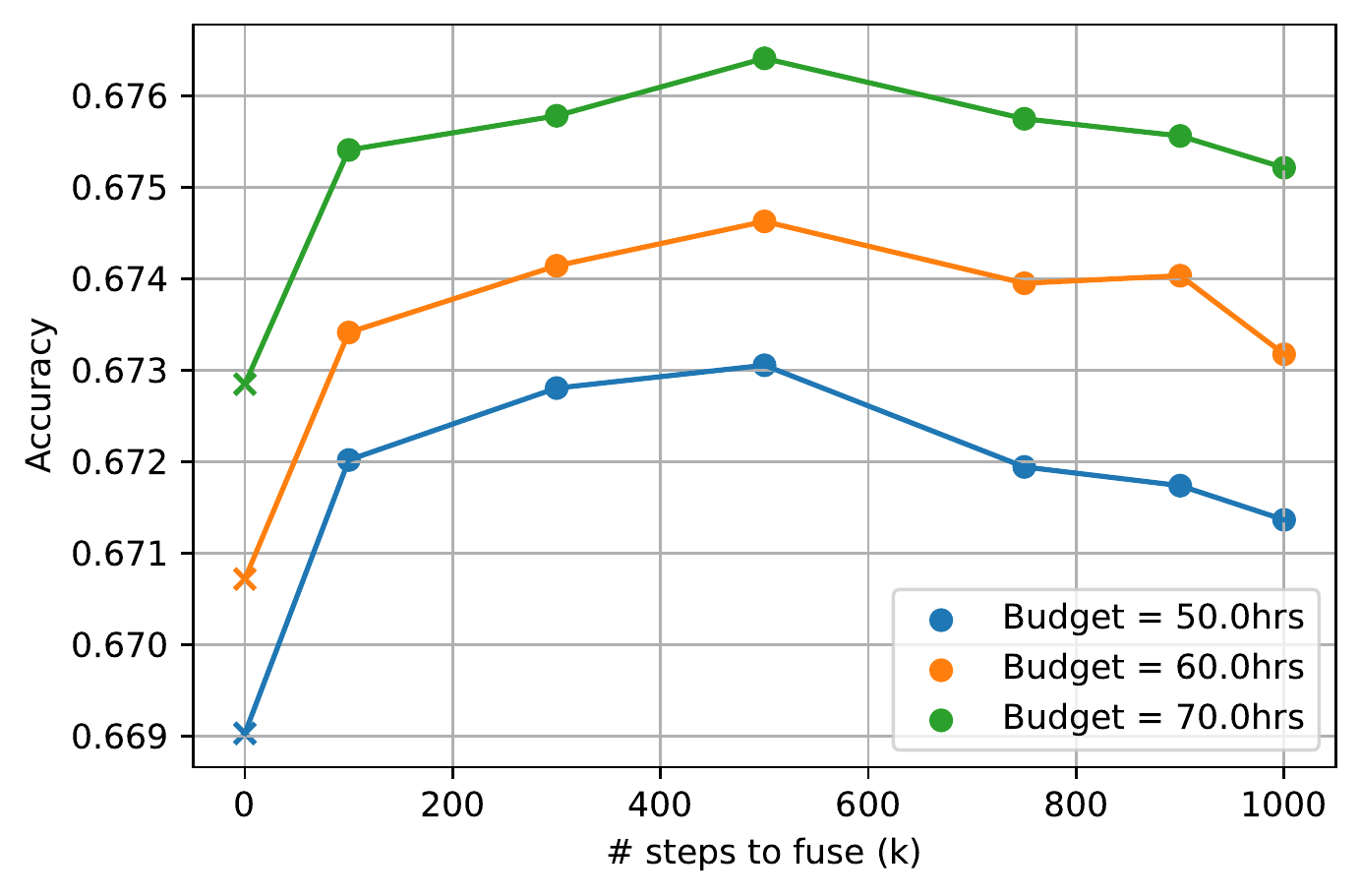}
\caption{Step to fuse. Displays the final accuracy of the fused model as a function step number when fusion was done (in 1000s). Each color represents a specific total budget allotted to training. The left-most points represent a randomly initialized baseline. \vspace{-0.5cm}}
\label{fig:time_to_fuse}
\end{figure}

\section{Discussion and Conclusion}\label{discussion}
In this paper, we present a new technique for improving the training process of large models. Our technique, called deep fusion, combines multiple models into a single model that can be trained more efficiently. We demonstrate how model fusion can be used to reduce the restrictions of distributed training, save on overall compute costs, and improve model performance.

Additionally, we introduced a new theoretical framework to illuminate the dynamics of mid-training network growth. This framework offers valuable insights to aid in the design and comprehension of network growing algorithms, easing the potential complexity they introduce when training already intricate large language models.
%\newtext{Additionally, we introduced a new theoretical framework to illuminate the dynamics of mid-training network growth, with insights to help design and understand network growing algorithm and ease the added complexity they might introduce when training an LLM - a process that is already very complex.}

In our experiments we fused models trained on the same data with identical architectures. While fusion has immediate training advantages, further research is needed to understand the implications and possible applications of fusing models trained on different sources and distinct architectures.

For example, it would be interesting to explore if transfer learning occurs when fusing models trained in different domains. Additional insights could arise by investigating the characteristics of models fused from submodels differing in dimensionality. For instance, one model could be attention-heavy, while another could be MLP-heavy. Finally, one could explore model fusion when the models are trained on different sequence lengths. This could also lead to efficiency improvements, as lower-length models train faster.
\vspace{-0.25cm}
%It would be interesting to explore if transfer learning occurs when fusing models trained in different domains, additionally, it would be interesting to understand characteristics of models that are fusion of models that differ in architecture dimensionality - for example, one that is attention heavy, and one that is feed forward heavy. Finally, it would be interesting to even explore fusion when the models are trained on different sequence lengths - this could also lead to efficiency improvements as lower length models train faster.

%We believe that model fusion is a promising technique for improving the training process of large models. We hope that our work will inspire further research in this area.

\section*{Impact statement}
This paper presents work whose goal is to advance the field of Machine Learning. There are many potential societal consequences of our work, none which we feel must be specifically highlighted here.

\section*{Acknowledgements}
The authors would like to thank Michael Munn for various stimulating conversations and insights throughout the
course of this work.

%%%%%%%%%%%%%%%%%%%%%%%%%%%%%%%bbl%%%%%%%%%%%%%%%%%%%%%%%%%%%%%%%%%%%%%%%%%%5
\bibliographystyle{plainnat}
\bibliography{references}

\begin{thebibliography}{42}
\providecommand{\natexlab}[1]{#1}
\providecommand{\url}[1]{\texttt{#1}}
\expandafter\ifx\csname urlstyle\endcsname\relax
  \providecommand{\doi}[1]{doi: #1}\else
  \providecommand{\doi}{doi: \begingroup \urlstyle{rm}\Url}\fi

\bibitem[Barba et~al.(2021)Barba, Jaggi, and Dandi]{barba2021federated}
Luis Barba, Martin Jaggi, and Yatin Dandi.
\newblock Implicit gradient alignment in distributed and federated learning.
\newblock In \emph{AAAI Conference on Artificial Intelligence}, AAAI'22, 2021.

\bibitem[Barrett and Dherin(2021)]{igr}
David G.~T. Barrett and Benoit Dherin.
\newblock Implicit gradient regularization.
\newblock In \emph{ICLR}, 2021.

\bibitem[Belkin et~al.(2019)Belkin, Hsu, Ma, and Mandal]{belkin2019reconciling}
Mikhail Belkin, Daniel Hsu, Siyuan Ma, and Soumik Mandal.
\newblock Reconciling modern machine-learning practice and the classical
  bias-variance trade-off.
\newblock \emph{Proceedings of the National Academy of Sciences}, 116\penalty0
  (32):\penalty0 15849--15854, 2019.

\bibitem[Brown et~al.(2020)Brown, Mann, Ryder, Subbiah, Kaplan, Dhariwal,
  Neelakantan, Shyam, Sastry, Askell, et~al.]{brown2020language}
Tom~B Brown, Benjamin Mann, Nick Ryder, Melanie Subbiah, Jared Kaplan, Prafulla
  Dhariwal, Arvind Neelakantan, Pranav Shyam, Girish Sastry, Amanda Askell,
  et~al.
\newblock Language models are few-shot learners.
\newblock In \emph{Advances in Neural Information Processing Systems}, pages
  14182--14193, 2020.

\bibitem[Cattaneo et~al.(2023)Cattaneo, Klusowski, and
  Shigida]{cattaneo2023implicit_bias_adam}
Matias Cattaneo, Jason Klusowski, and Boris Shigida.
\newblock On the implicit bias of adam.
\newblock \emph{arXiv:2309.00079}, 2023.

\bibitem[Chowdhery et~al.(2022)Chowdhery, Narang, Devlin, Bosma, Mishra,
  Roberts, Barham, Chung, Sutton, Gehrmann, Schuh, Shi, Tsvyashchenko, Maynez,
  Rao, Barnes, Tay, Shazeer, Prabhakaran, Reif, Du, Hutchinson, Pope, Bradbury,
  Austin, Isard, Gur-Ari, Yin, Duke, Levskaya, Ghemawat, Dev, Michalewski,
  Garcia, Misra, Robinson, Fedus, Zhou, Ippolito, Luan, Lim, Zoph, Spiridonov,
  Sepassi, Dohan, Agrawal, Omernick, Dai, Pillai, Pellat, Lewkowycz, Moreira,
  Child, Polozov, Lee, Zhou, Wang, Saeta, Diaz, Firat, Catasta, Wei,
  Meier-Hellstern, Eck, Dean, Petrov, and Fiedel]{chowdhery2022palm}
Aakanksha Chowdhery, Sharan Narang, Jacob Devlin, Maarten Bosma, Gaurav Mishra,
  Adam Roberts, Paul Barham, Hyung~Won Chung, Charles Sutton, Sebastian
  Gehrmann, Parker Schuh, Kensen Shi, Sasha Tsvyashchenko, Joshua Maynez,
  Abhishek Rao, Parker Barnes, Yi~Tay, Noam Shazeer, Vinodkumar Prabhakaran,
  Emily Reif, Nan Du, Ben Hutchinson, Reiner Pope, James Bradbury, Jacob
  Austin, Michael Isard, Guy Gur-Ari, Pengcheng Yin, Toju Duke, Anselm
  Levskaya, Sanjay Ghemawat, Sunipa Dev, Henryk Michalewski, Xavier Garcia,
  Vedant Misra, Kevin Robinson, Liam Fedus, Denny Zhou, Daphne Ippolito, David
  Luan, Hyeontaek Lim, Barret Zoph, Alexander Spiridonov, Ryan Sepassi, David
  Dohan, Shivani Agrawal, Mark Omernick, Andrew~M. Dai,
  Thanumalayan~Sankaranarayana Pillai, Marie Pellat, Aitor Lewkowycz, Erica
  Moreira, Rewon Child, Oleksandr Polozov, Katherine Lee, Zongwei Zhou, Xuezhi
  Wang, Brennan Saeta, Mark Diaz, Orhan Firat, Michele Catasta, Jason Wei,
  Kathy Meier-Hellstern, Douglas Eck, Jeff Dean, Slav Petrov, and Noah Fiedel.
\newblock Palm: Scaling language modeling with pathways, 2022.

\bibitem[Devlin et~al.(2018)Devlin, Chang, Lee, and Toutanova]{devlin2018bert}
Jacob Devlin, Ming-Wei Chang, Kenton Lee, and Kristina Toutanova.
\newblock Bert: Pre-training of deep bidirectional transformers for language
  understanding.
\newblock In \emph{Proceedings of the 2019 Conference of the North American
  Chapter of the Association for Computational Linguistics: Human Language
  Technologies}, pages 4171--4186, 2018.

\bibitem[Dherin(2023)]{dherin2023implicit}
Benoit Dherin.
\newblock Implicit biases in multitask and continual learningfrom a backward
  error analysis perspective.
\newblock In \emph{NeurIPS, Mathematics of Modern Machine Learning Workshop},
  2023.

\bibitem[Frankle and Carbin(2018)]{frankle2018lottery}
Jonathan Frankle and Michael Carbin.
\newblock The lottery ticket hypothesis: Finding sparse, trainable neural
  networks.
\newblock In \emph{Proceedings of the International Conference on Learning
  Representations}, 2018.

\bibitem[Frankle and Carbin(2019)]{frankle2019lottery}
Jonathan Frankle and Michael Carbin.
\newblock The lottery ticket hypothesis at scale.
\newblock In \emph{arXiv preprint arXiv:1903.01611}, 2019.

\bibitem[Ganesh et~al.(2021)Ganesh, Chen, Lou, Khan, Yang, Sajjad, Nakov, Chen,
  and Winslett]{Ganesh21compress}
Prakhar Ganesh, Yao Chen, Xin Lou, Mohammad~Ali Khan, Yin Yang, Hassan Sajjad,
  Preslav Nakov, Deming Chen, and Marianne Winslett.
\newblock {Compressing Large-Scale Transformer-Based Models: A Case Study on
  BERT}.
\newblock \emph{Transactions of the Association for Computational Linguistics},
  9:\penalty0 1061--1080, 09 2021.
\newblock ISSN 2307-387X.
\newblock \doi{10.1162/tacl_a_00413}.
\newblock URL \url{https://doi.org/10.1162/tacl\_a\_00413}.

\bibitem[Gao et~al.(2023)Gao, Zhibong, Zhou, Kang, and
  Chaudhari]{gao2023diffusion_bea}
Yansong Gao, Pan Zhibong, Xin Zhou, Le~Kang, and Pratik Chaudhari.
\newblock Fast diffusion probabilistic model sapling through the lens of
  backward error analysis.
\newblock \emph{arXiv:2304.11446}, 2023.

\bibitem[Ghosh et~al.(2023)Ghosh, Lyu, Zhang, and Wang]{ghosh2023implicit}
Avrajit Ghosh, He~Lyu, Xitong Zhang, and Rongrong Wang.
\newblock Implicit regularization in heavy-ball momentum accelerated stochastic
  gradient descent.
\newblock \emph{ICLR}, 2023.

\bibitem[Glorot and Bengio(2010)]{glorot2010understanding}
Xavier Glorot and Yoshua Bengio.
\newblock Understanding the difficulty of training deep feedforward neural
  networks.
\newblock In \emph{Proceedings of the Thirteenth International Conference on
  Artificial Intelligence and Statistics}, pages 249--256, 2010.

\bibitem[Gong et~al.(2019)Gong, He, Li, Qin, Wang, and Liu]{bert_pstacking}
L.~Gong, D.~He, Z.~Li, T.~Qin, L.~Wang, and T.~Liu.
\newblock Efficient training of bert by progressively stacking.
\newblock In \emph{ICML 2019}, 2019.
\newblock URL \url{https://proceedings.mlr.press/v97/gong19a.html}.

\bibitem[Graves(2016)]{graves2016adaptive}
Alex Graves.
\newblock Adaptive computation time for recurrent neural networks.
\newblock In \emph{Proceedings of the International Conference on Learning
  Representations}, 2016.

\bibitem[Gu et~al.(2021{\natexlab{a}})Gu, Hu, Zhao, Lin, Cheng, Wang, and
  Wan]{gu2021palm}
Jiatao Gu, Jianqiang Hu, Tong Zhao, Ying Lin, Xiuying Cheng, Lijun Wang, and
  Xiang Wan.
\newblock Palm: Pre-training an autoencoding and autoregressive language model
  for context-conditioned generation.
\newblock In \emph{Proceedings of the 2021 Conference on Empirical Methods in
  Natural Language Processing}, pages 2643--2662, 2021{\natexlab{a}}.

\bibitem[Gu et~al.(2021{\natexlab{b}})Gu, Liu, Yu, Li, Chen, and
  Han]{transformer_grow_p_bert}
Xiaotao Gu, Liyuan Liu, Hongkun Yu, Jing Li, Chen Chen, and Jiawei Han.
\newblock On the transformer growth for progressive bert training.
\newblock In \emph{In Proceedings of the 2021 Conference of the North American
  Chapter of the Association for Computational Linguistics: Human Language
  Technologies}, page 5174–5180, 2021{\natexlab{b}}.

\bibitem[Hairer et~al.(2006)Hairer, Lubich, and Wanner]{hairer2006}
Ernst Hairer, Christian Lubich, and Gerhard Wanner.
\newblock \emph{Geometric numerical integration}, volume~31 of \emph{Springer
  Series in Computational Mathematics}.
\newblock Springer-Verlag, Berlin, second edition, 2006.
\newblock ISBN 3-540-30663-3; 978-3-540-30663-4.
\newblock Structure-preserving algorithms for ordinary differential equations.

\bibitem[He et~al.(2015)He, Zhang, Ren, and Sun]{he2015delving}
Kaiming He, Xiangyu Zhang, Shaoqing Ren, and Jian Sun.
\newblock Delving deep into rectifiers: Surpassing human-level performance on
  imagenet classification.
\newblock In \emph{Proceedings of the IEEE International Conference on Computer
  Vision}, pages 1026--1034, 2015.

\bibitem[Kaddour et~al.(2023)Kaddour, Key, Nawrot, Minervini, and
  Kusner]{refute_growing}
Jean Kaddour, Oscar Key, Piotr Nawrot, Pasquale Minervini, and Matt~J. Kusner.
\newblock No train no gain: Revisiting efficient training algorithms for
  transformer-based language models.
\newblock In \emph{Neurips 2023}, 2023.
\newblock URL \url{https://arxiv.org/pdf/2307.06440.pdf}.

\bibitem[Kaplan et~al.(2020)Kaplan, McCandlish, Henighan, Brown, Chess, Child,
  Gray, Radford, Wu, and Amodei]{kaplan2020scaling}
Jared Kaplan, Sam McCandlish, Tom Henighan, Tom~B Brown, Benjamin Chess, Rewon
  Child, Scott Gray, Alec Radford, Jeffrey Wu, and Dario Amodei.
\newblock Scaling laws for neural language models.
\newblock \emph{arXiv preprint arXiv:2001.08361}, 2020.

\bibitem[Kwon et~al.(2022)Kwon, Kim, Mahoney, Hassoun, Keutzer, and
  Gholami]{kwon2022a}
Woosuk Kwon, Sehoon Kim, Michael~W. Mahoney, Joseph Hassoun, Kurt Keutzer, and
  Amir Gholami.
\newblock A fast post-training pruning framework for transformers.
\newblock In Alice~H. Oh, Alekh Agarwal, Danielle Belgrave, and Kyunghyun Cho,
  editors, \emph{Advances in Neural Information Processing Systems}, 2022.
\newblock URL \url{https://openreview.net/forum?id=0GRBKLBjJE}.

\bibitem[Lee(2022)]{lee2012smooth_manifolds}
John Lee.
\newblock Introduction to smooth manifolds.
\newblock In \emph{Springer}, 2022.

\bibitem[Li et~al.(2023)Li, Yao, Jiang, Fang, Meng, Fan, Han, Li, Du, Qin,
  Zhang, Sun, and Wang]{100_k_dollars}
Xiang Li, Yiqun Yao, Xin Jiang, Xuezhi Fang, Xuying Meng, Siqi Fan, Peng Han,
  Jing Li, Li~Du, Bowen Qin, Zheng Zhang, Aixin Sun, and Yequan Wang.
\newblock Flm-101b: An open llm and how to train it with \$100k budget.
\newblock In \emph{Arxiv}, 2023.
\newblock URL \url{https://arxiv.org/pdf/2309.03852.pdf}.

\bibitem[Micikevicius et~al.(2017)Micikevicius, Narang, Alben, Diamos, Elsen,
  Garcia, Ginsburg, Houston, Kuchaiev, Venkatesh,
  et~al.]{micikevicius2017mixed}
Paulius Micikevicius, Sharan Narang, Jonah Alben, Gregory Diamos, Erich Elsen,
  David Garcia, Boris Ginsburg, Michael Houston, Oleksii Kuchaiev, Ganesh
  Venkatesh, et~al.
\newblock Mixed precision training.
\newblock In \emph{Proceedings of the International Conference on Learning
  Representations}, 2017.

\bibitem[Nakkiran et~al.(2020)Nakkiran, Kaplun, Bansal, Yang, Barak, and
  Sutskever]{nakkiran2020deep}
Preetum Nakkiran, Gal Kaplun, Yamini Bansal, Tristan Yang, Boaz Barak, and Ilya
  Sutskever.
\newblock Deep double descent: Where bigger models and more data hurt.
\newblock In \emph{Proceedings of the International Conference on Learning
  Representations}, 2020.

\bibitem[Narayanan et~al.(2021)Narayanan, Shoeybi, Casper, LeGresley, Patwary,
  Korthikanti, Vainbrand, Kashinkunti, Bernauer, Catanzaro, Phanishayee, and
  Zaharia]{Narayanan21large_scale_llm}
Deepak Narayanan, Mohammad Shoeybi, Jared Casper, Patrick LeGresley, Mostofa
  Patwary, Vijay Korthikanti, Dmitri Vainbrand, Prethvi Kashinkunti, Julie
  Bernauer, Bryan Catanzaro, Amar Phanishayee, and Matei Zaharia.
\newblock Efficient large-scale language model training on gpu clusters using
  megatron-lm.
\newblock In \emph{Proceedings of the International Conference for High
  Performance Computing, Networking, Storage and Analysis}, SC '21, New York,
  NY, USA, 2021. Association for Computing Machinery.
\newblock ISBN 9781450384421.
\newblock \doi{10.1145/3458817.3476209}.
\newblock URL \url{https://doi.org/10.1145/3458817.3476209}.

\bibitem[Radford et~al.(2019)Radford, Narasimhan, Salimans, and
  Sutskever]{radford2019language}
Alec Radford, Karthik Narasimhan, Tim Salimans, and Ilya Sutskever.
\newblock Language models are unsupervised multitask learners.
\newblock In \emph{OpenAI Blog}, 2019.

\bibitem[Raffel et~al.(2019)Raffel, Shazeer, Roberts, Lee, Narang, Matena,
  Zhou, Li, and Liu]{t5paper}
Colin Raffel, Noam Shazeer, Adam Roberts, Katherine Lee, Sharan Narang, Michael
  Matena, Yanqi Zhou, Wei Li, and Peter~J. Liu.
\newblock Exploring the limits of transfer learning with a unified text-to-text
  transformer.
\newblock \emph{CoRR}, 2019.
\newblock URL \url{http://arxiv.org/abs/1910.10683}.

\bibitem[Rosca et~al.(2021)Rosca, Wu, Dherin, and
  Barrett]{rosca2021discretization}
Mihaela Rosca, Yan Wu, Benoit Dherin, and David~G.T. Barrett.
\newblock Discretization drift in two-player games.
\newblock In \emph{ICML}, 2021.

\bibitem[Rusu et~al.(2016)Rusu, Rabinowitz, Desjardins, Soyer, Kirkpatrick,
  Kavukcuoglu, Pascanu, and Hadsell]{rusu2016progressive}
Andrei~A Rusu, Neil~C Rabinowitz, Guillaume Desjardins, Hubert Soyer, James
  Kirkpatrick, Koray Kavukcuoglu, Razvan Pascanu, and Raia Hadsell.
\newblock Progressive neural networks.
\newblock In \emph{Proceedings of the 30th International Conference on Neural
  Information Processing Systems}, pages 2212--2220, 2016.

\bibitem[Shen et~al.(2022)Shen, Walsh, Keutzer, Dodge, Peters, and
  Beltagy]{shen2022staged}
Sheng Shen, Pete Walsh, Kurt Keutzer, Jesse Dodge, Matthew Peters, and
  Iz~Beltagy.
\newblock Staged training for transformer language models.
\newblock In Kamalika Chaudhuri, Stefanie Jegelka, Le~Song, Csaba Szepesvari,
  Gang Niu, and Sivan Sabato, editors, \emph{Proceedings of the 39th
  International Conference on Machine Learning}, volume 162 of
  \emph{Proceedings of Machine Learning Research}, pages 19893--19908. PMLR,
  17--23 Jul 2022.
\newblock URL \url{https://proceedings.mlr.press/v162/shen22f.html}.

\bibitem[Smith et~al.(2021)Smith, Dherin, Barrett, and De]{smith2021on}
Samuel~L Smith, Benoit Dherin, David~G.T. Barrett, and Soham De.
\newblock On the origin of implicit regularization in stochastic gradient
  descent.
\newblock In \emph{ICLR}, 2021.

\bibitem[Touvron et~al.(2023)Touvron, Lavril, Izacard, Martinet, Lachaux,
  Lacroix, Rozière, Goyal, Hambro, Azhar, Rodriguez, Joulin, Grave, and
  Lample]{touvron2023llama}
Hugo Touvron, Thibaut Lavril, Gautier Izacard, Xavier Martinet, Marie-Anne
  Lachaux, Timothée Lacroix, Baptiste Rozière, Naman Goyal, Eric Hambro,
  Faisal Azhar, Aurelien Rodriguez, Armand Joulin, Edouard Grave, and Guillaume
  Lample.
\newblock Llama: Open and efficient foundation language models, 2023.

\bibitem[Wei et~al.(2016)Wei, Wang, Rui, and Chen]{wei2016network}
Chen Wei, Haoyu Wang, Yongyang Rui, and Changqing Chen.
\newblock Network morphism.
\newblock In \emph{Proceedings of the 33rd International Conference on
  International Conference on Machine Learning}, pages 2662--2671, 2016.

\bibitem[Yao et~al.(2023)Yao, Zhang, Li, and Wang]{twice_faster}
Yiqun Yao, Zheng Zhang, Jing Li, and Yequan Wang.
\newblock 2x faster language model pre-training via masked structural growth.
\newblock In \emph{Arxiv}, 2023.
\newblock URL \url{https://arxiv.org/pdf/2305.02869.pdf}.

\bibitem[You et~al.(2017{\natexlab{a}})You, Bulatov, Gitman, and
  Risteski]{you2017large}
Yang You, Yaroslav Bulatov, Igor Gitman, and Andrej Risteski.
\newblock Large batch training of convolutional networks.
\newblock In \emph{arXiv preprint arXiv:1708.03888}, 2017{\natexlab{a}}.

\bibitem[You et~al.(2017{\natexlab{b}})You, Gitman, and
  Ginsburg]{you2017scaling}
Yang You, Igor Gitman, and Boris Ginsburg.
\newblock Scaling sgd batch size to 32k for imagenet training.
\newblock In \emph{Deep Learning Scaling is Predictable, Empirically}, page~13,
  2017{\natexlab{b}}.

\bibitem[Zhang et~al.(2020)Zhang, Zhang, Ghosh, Li, Tasci, Heck, Zhang, and
  Kuo]{zhang2020classincremental}
Junting Zhang, Jie Zhang, Shalini Ghosh, Dawei Li, Serafettin Tasci, Larry
  Heck, Heming Zhang, and C.~C.~Jay Kuo.
\newblock Class-incremental learning via deep model consolidation, 2020.

\bibitem[Zhang et~al.(2022)Zhang, Zuo, Liang, Bukharin, He, Chen, and
  Zhao]{zhang2022platon}
Qingru Zhang, Simiao Zuo, Chen Liang, Alexander Bukharin, Pengcheng He, Weizhu
  Chen, and Tuo Zhao.
\newblock Platon: Pruning large transformer models with upper confidence bound
  of weight importance, 2022.

\bibitem[Zhou et~al.(2023)Zhou, Liu, Xu, Iyer, Sun, Mao, Ma, Efrat, Yu, Yu,
  Zhang, Ghosh, Lewis, Zettlemoyer, and Levy]{zhou2023lima}
Chunting Zhou, Pengfei Liu, Puxin Xu, Srini Iyer, Jiao Sun, Yuning Mao, Xuezhe
  Ma, Avia Efrat, Ping Yu, Lili Yu, Susan Zhang, Gargi Ghosh, Mike Lewis, Luke
  Zettlemoyer, and Omer Levy.
\newblock Lima: Less is more for alignment, 2023.

\end{thebibliography}

\begin{table*}[!t]
\centering
\begin{tiny}
\begin{tabularx}{\textwidth}{cXXXXXXXXXXXXX}
Model & Glue avg & COLA Matthew's & SST acc & MRPC f1 & MRPC acc & STS-b pearson & STS-b \mbox{spearman} & qqp acc & qqp f1 & MNLI-m & MNLI-mm & QNLI & RTE\\
\toprule
baseline 1m & 84.74 & 54.18 & 94.38 & 93.17 & 90.69 & 89.83 & 89.75	& 91.94 & 89.13 & 86.77 & 86.6 & 92.13 & 78.34 \\
baseline 1m & 84.45 & 52.95 & 93.92 & 93.12 & 90.44 & 89.49 & 89.35 & 91.94 & 89.14 & 86.94 & 86.58 & 92.26 & 77.98 \\
baseline 1m & 84.69 & 53.15 & 93.81 & 92.15 & 88.97 & 89.06 & 88.93 & 92.04 & 89.24 & 86.55 & 86.21 & 91.69 & 82.31\\ \midrule
Fusion 500k & 84.98 & 53.97 & 94.27 & 92.91 & 90.2 & 89.68 & 89.49 & 92.04 & 89.36 & 86.6 & 86.43 & 92.39 & 80.87 \\
Fusion 500k & 84.68 & 54.46 & 93.81 & 91.8 & 88.97 & 89.44 & 89.28 & 91.95 & 89.13 & 86.67 & 86.65 & 92.11 & 80.14 \\
Fusion 500k & 84.92 & 53.94 & 94.5 & 92.73 & 89.71 & 90.62 & 90.44 & 92.01 & 89.23 & 86.64 & 86.45 & 91.56 & 80.51 \\ \midrule
Self fusion 500K & 84.69 & 54.58 & 93.12 & 92.03 & 89.22 & 89.23 & 89.16 & 91.81 & 89.01 & 86.64 & 86.68 & 92.11 & 80.87 \\
Self fusion 500K & 85.19 & 55.82 & 93.69 & 93.1 & 90.44 & 89.9 & 89.73 & 92.04 & 89.34 & 86.8 & 86.31 & 92.31 & 80.87 \\
Self fusion 500K & 84.95 & 58 & 94.27 & 92.36 & 89.71 & 89.93 & 89.7 & 91.96 & 89.18 & 86.47 & 86.4 & 91.93 & 77.62 \\ \bottomrule
\end{tabularx}
\end{tiny}
\caption{Performance (Glue tasks) of the various models on downstream tasks.}
\label{tbl:full_glue}
\end{table*}

%%%%%%%%%%%%%%%%%%%%%%%%%%%%%%%%%%%%%%%%%%%%%%%%%%%%%%%%%%%%%%%%%%%%%%%%%%%%%%%%%%%%%%%%%%%%%%%%%%%%
%%%%%%%%%%%%%%%%%%%%%%%%%%%%%%%%%%%%%%%%%%%%%%%%%%%%%%%%%%%%%%%%%%%%%%%%%%%%%%%%%%%%%%%%%%%%%%%%%%%%

\newpage
\appendix

\section{Proofs }
\label{ap:1}

In this appendix, we provide proofs for Theorem~\ref{th:modified_loss} and Theorem~\ref{th:modified_equation}, which are restated below for convenience.

\marginbound*
\modifiedEquation*

\begin{proof}

We are expanding the second part of Eq.~\ref{eq:w_2} using a Taylor series expansion around $w_0$:
\begin{align*}
    w_2 = & w_0 + hF(w_0) + \alpha h G(w_0 + hF(w_0)) = \\
    = & w_0 + h(F(w_0) + \alpha G(w_0)) + \\
    & h^2\alpha\nabla G(w_0)F(w_0) + {\cal O}(h^3).
\end{align*}

We want a modified equation of the form:
\begin{equation}
\label{eq:modified_w}
    \dot{w} = H_0(w) + hH_1(w) + h^2 H_2(w) + \ldots.
\end{equation}

The Taylor series expansion of the modified equation solution is:
\begin{equation}
\label{eq:cont_w}
    w(h) = w + h\dot{w} + \frac{h^2}{2} \nabla \dot{w}\dot{w} + \ldots.
\end{equation}

Substituting Eq.~\ref{eq:modified_w} into Eq.~\ref{eq:cont_w}:
\begin{align*}
    w(h) = & w + h(H_0(w)+hH_1(w)) + \\
    & \frac{h^2}{2}H'_0(w)H_0(w)+{\cal O}(h^3) \\
    = & w + hH_0(w) + {\cal O}(h^3) + \\ 
    & h^2\left( H_1(w) + \frac{1}{2} H'_0(w)H_0(w) \right) 
\end{align*}

To have $w_1 = w(h)$ at first order, we obtain the following condition: $H_0(w) = F(w_0) + \alpha G(w_0)$ and,
\begin{align*}
  H_1(w) &+ \frac{1}{2}H'_0(w)H_0(w) = \alpha \nabla G(w_0)F(w_0) \\  % new line
  H_1(w) &+ \frac{1}{2}(\nabla F(w_0) + \alpha\nabla G(w_0))( F(w_0) + \alpha G(w_0) ) = \\ &  \alpha \nabla G(w_0)F(w_0) \\
  H_1(w) & = \frac{\alpha}{2} ( \nabla G(w_0)F(w_0) - \nabla F(w_0)G(w_0) ) - \\ & \frac{1}{2} ( \nabla F(w_0)F(w_0) + \alpha^2 \nabla G(w_0)G(w_0) ).
\end{align*}

Using the definition of $F$ and $G$:
\begin{align*}
    \nabla F(w_0)F(w_0) = & \nabla_w( \nabla_w L_1(w_0) ) \nabla_w L_1(w_0) =\\
    = & \frac{\nabla_w}{2} \lVert \nabla_w L_1(w_0) \rVert^2.
\end{align*}
then,
\begin{align*}
    H_1(w) = & \frac{\alpha}{2}[ \nabla_w L_1, \nabla_w L_2 ](w_0) -\\
    & \nabla_w \left( \frac{1}{4} \lVert \nabla_w L_1(w_0) \rVert^2 + \frac{\alpha^2}{4} \lVert \nabla_w L_2(w_0) \rVert^2 \right).
\end{align*}

So the continuous modified equation is:
\begin{align*}
    \dot{w}(t) = & -\nabla_w \left( L_1(w_t) + \alpha L_2(w_t) + \frac{h}{4} \lVert \nabla_w L_1(w_t) \rVert^2 + \right. \\
    &\left. \frac{h\alpha^2}{4} \lVert \nabla_w L_2(w_t) \rVert^2 \right) + \\
    & \frac{h\alpha}{2}[ \nabla_w L_1, \nabla_w L_2 ](w_t) + {\cal O}(h^2).
\end{align*}

We can write the modified loss by substituting $L_1(w)$ and $L_2(w)$ by their definition using $L_S$ and $L_B$ to obtain,
\begin{align*}
    \tilde{L}(w) &= L \left( \frac{1}{n}\sum_{i=1}^{n} f_i(x;\theta_i), y \right) + L_B(w) + \\
    & \frac{h}{4} \lVert \nabla_{\theta_i} L\left( \frac{1}{n}\sum_{i=1}^{n} f_i(x;\theta_i), y \right) \rVert^2 +\\
    & \frac{h\alpha^2}{4} \lVert \nabla_w L_B(w) \rVert^2.
\end{align*}

Applying Jensen's inequality to $L \left( \frac{1}{n}\sum_{i=1}^{n} f_i(x;\theta_i), y \right) < \frac{1}{n} \sum_{i=1}^{n} L(\theta_i)$, where $L(\theta_i) = L_S(\theta_i)$ finalizes the proof.

\end{proof}

\scaledGradient*

\begin{proof}
We present the proof for self deep fusion of a model with itself, and the expansion to $n$ copies can be proved in the same manner. 
First, it is easy to see that when $\theta_1 = \theta_2$ we have 
$$
\nabla_{\theta_1} L_B(\theta_1, \theta_2, \eta=0) = \nabla_{\theta_2} L_B(\theta_1, \theta_2, \eta=0).
$$
This follows immediately in self deep fusion from the symmetry of the network, and the fact that the model is copied twice as is.

Next, we will show that 
\begin{align*}
\nabla_{\theta_2} L_S(\theta_2) &= \nabla_{\theta_1} L_S(\theta_1) = \nabla_{\theta_1} L_B(\theta_1, \theta_2, \eta=0) + \\ &\nabla_{\theta_2} L_B(\theta_1, \theta_2, \eta=0).
\end{align*}
To see this note that when $\eta=0$,
{\footnotesize
\begin{align}
&\nabla_{\theta_{1_j}} L_S(\theta_1) = \lim_{\epsilon \rightarrow 0}\frac{ L_S(\theta_1 + \mathbf{1}_j\cdot \epsilon)- L_S(\theta_1)}{\epsilon} \nonumber \\
&= \lim_{\epsilon \rightarrow 0} \frac{L_B(\theta_1 + \mathbf{1}_j\cdot \epsilon, \theta_2+ \mathbf{1}_j\cdot \epsilon, \eta) -L_B(\theta_1, \theta_2, \eta)  }{\epsilon} \nonumber \\
&=\lim_{\epsilon \rightarrow 0} \frac{1}{\epsilon}\biggl[L_B(\theta_1, \theta_2+ \mathbf{1}_j\cdot \epsilon, \eta) + \epsilon \nabla_{\theta_{1_j}}L_B(\theta_1, \theta_2+ \mathbf{1}_j\cdot \epsilon, \eta) \nonumber \\
&\ \ \ \ \ \ \ \ \ \ + O(\epsilon^2) - L_B(\theta_1, \theta_2, \eta)  \biggr] \nonumber \\
&=\lim_{\epsilon \rightarrow 0} \frac{1}{\epsilon}\biggl[(L_B(\theta_1, \theta_2, \eta) + \epsilon \nabla_{\theta_{2_j}}L_B(\theta_1, \theta_2, \eta) + \nonumber \\ 
&\ \ \ \ \ \ \ \ \ \ + \epsilon \nabla_{\theta_{1_j}}L_B(\theta_1, \theta_2+ \mathbf{1}_j\cdot \epsilon, \eta) + O(\epsilon^2) - L_B(\theta_1, \theta_2, \eta)  \biggr] \nonumber \\
&=\lim_{\epsilon \rightarrow 0} \frac{1}{\epsilon}\biggl[L_B(\theta_1, \theta_2, \eta) + \epsilon \nabla_{\theta_{2_j}}L_B(\theta_1, \theta_2, \eta) + \nonumber \\ 
&\ \ \ \ \ \ \ \ \ \ + \epsilon \nabla_{\theta_{1_j}}L_B(\theta_1, \theta_2, \eta) + O(\epsilon^2) - L_B(\theta_1, \theta_2, \eta) \biggr] \nonumber \\
&=\nabla_{\theta_{1_j}} L_B(\theta_1, \theta_2, \eta) + \nabla_{\theta_{2_j}}L_B(\theta_1, \theta_2, \eta) 
\end{align}
}

% To see this, we note the following: Denote by $f'$ the function~$f$ with one parameter $p$ shifted by $\epsilon$. 
% \begin{align}
% \nabla_{p}& L_S(w) = \lim_{\epsilon \rightarrow 0} \frac{f' - f}{\epsilon} \nonumber \\
% &=\lim_{\epsilon \rightarrow 0} \frac{DF(f', f') - DF(f, f)}{\epsilon} \nonumber \\
% &=\lim_{\epsilon \rightarrow 0} \frac{DF(f, f') + \epsilon \frac{\delta DF(f, f')}{\delta p_1} + O(\epsilon^2) - DF(f, f)}{\epsilon} \nonumber \\
% &=\lim_{\epsilon \rightarrow 0} \frac{DF(f, f) +\epsilon \frac{\delta DF(f, f)}{\delta p_2} + \epsilon \frac{\delta DF(f, f')}{\delta p_1} + O(\epsilon^2) - DF(f, f)}{\epsilon} \nonumber \\
% &=\lim_{\epsilon \rightarrow 0} \frac{\epsilon \frac{\delta DF(f, f)}{\delta p_2} + \epsilon \frac{\delta DF(f, f')}{\delta p_1} + O(\epsilon^2)}{\epsilon} \nonumber \\
% &=\lim_{\epsilon \rightarrow 0} \frac{\epsilon \frac{\delta DF(f, f)}{\delta p_2} + \epsilon  \frac{\delta DF(f, f)}{\delta p_1} + O(\epsilon^2)}{\epsilon} \nonumber \\
% &= \nabla_{p_1} L_B(\theta_1) + \nabla_{p_2} L_B(\theta_2).
% \end{align}
In the above, $\mathbf{1}_j$ is the unity vector with $j$ being the only non-zero entry, and we use Taylor expansion to move from one line to the other. Additionally, since the functions in question have Lipschitz continuity, then whatever is in the $O(\epsilon^2)$ can be bounded by a constant times $\epsilon^2$.
\end{proof}

\sameGradient*

\begin{proof}
As before we prove the lemma on fusing the same model twice for simplicity. Extending to $n$ copies can be proved in the same manner. To prove that every parameter in $\eta$ receive the same gradients as some parameter in~$\theta$ in the same layer we use the following two observations:
\begin{itemize}
    \item The hidden representation in every layer is a vector of the form $(y, y)$ for some real number vector $y$ (Fusion Property). This means that the {\it input} to every layer is of the form $(y, y)$ for some real number vector $y$.
    \item Given the above, each parameter in $\eta$ affects the prediction in same way as one of the parameters in $\theta_i$, which means adding $\epsilon$ either to that variable, or to the equivalent variable in $\theta_i$ has the same effect on the predictions.
\end{itemize}

More formally, it easy to see that:
$$
(y || y)\begin{pmatrix}\theta + \epsilon \mathbf{1}_{i,j} & \mathbf{0} \\ \mathbf{0} & \theta \end{pmatrix} = (y || y)\begin{pmatrix}\theta & \mathbf{0} \\ \mathbf{0} + \epsilon_1 \mathbf{1}_{i,j} & \theta \end{pmatrix} 
$$
where $\mathbf{1}_{i,j}$ is a matrix with zeros in every entry except entry~${i,j}$ that is equal to 1.

By definition the gradient for a parameter is the change in loss when applying $\epsilon$ change to the parameter (divided by~$\epsilon$). From the above it is easy to see that the gradient for the two parameters is equivalent as they affect the Loss in the same exact way. This together with Lemma~\ref{lemma:scaled_gradient} concludes the proof.
\end{proof}
\begin{corollary}
From Lemma~\ref{lemma:scaled_gradient} and Lemma~\ref{lemma:same_gradient} given a layer with parameters $\theta$, the fused layer gradient looks as follows  
$$
\nabla_{\theta,\theta,\eta} L_B = \begin{pmatrix} \frac{1}{2} \nabla_\theta L_s & \frac{1}{2} \nabla_\theta L_s \\ \frac{1}{2} \nabla_\theta L_s & \frac{1}{2} \nabla_\theta L_s \end{pmatrix},
$$

\end{corollary}
which concludes the proof.
%The proof can be done by induction, starting from the last layer, and moving deeper in the network. Denote by $h_i$ the layer $i$ in the small model, and let $r$ be the number of layers.

% We have that 
% $$DF(f,f)([x,x]) =  \textsc{Avg}(h_r \rawfusion h_r )([y, y]),$$ where $y$ is the input for the last layer. It is easy to see that 
% $$
% \textsc{Avg}(h_r \rawfusion h_r )([y + \alpha, y]) = \textsc{Avg}(h_r \rawfusion h_r )([y , y + \alpha]),
% $$
% for any vector $\alpha$. It follows immediately from the fact that we are fusing the same layer, and taking average (symmetry). With induction, it is easy prove that for deeper layers, $h_i$ we have that

% For the last layer, when performing self deep fusion of a model function $f$, we are taking the average of 
% $$
% DF(f, f) = \textsc{Avg}(f_1 \rawfusion f_2)
% $$
% where $f_1$ and $f_2$ are copes of $f$ with parameters $\theta_1$ and $\theta_2$, respectively. For the last layer, we notice two things:
% \begin{itemize}
%     \item $f(x) = DF(f, f)(x) = \textsc{Avg}(f_1 \rawfusion f_2)(x)$ for all $x$.
%     \item The input to the last layer is a vector of the form $(y || y)$ for some vector $y$.
% \end{itemize} 
% Combining the two observation above, it is easy to see that 
% $$
% \nabla_{p_1} L_B(w) = \nabla_{p_2} L_B(w) = \frac 12 \nabla_{p} L_S(\theta),
% $$
% for any parameter $p_1$, $p_2$ are copies of the parameter $p$, and~$p$ is any parameter in the last layer.
% To complete the proof, we go one layer deeper. Let 

\begin{figure*}[t]
\centering
\includegraphics[scale=0.18]{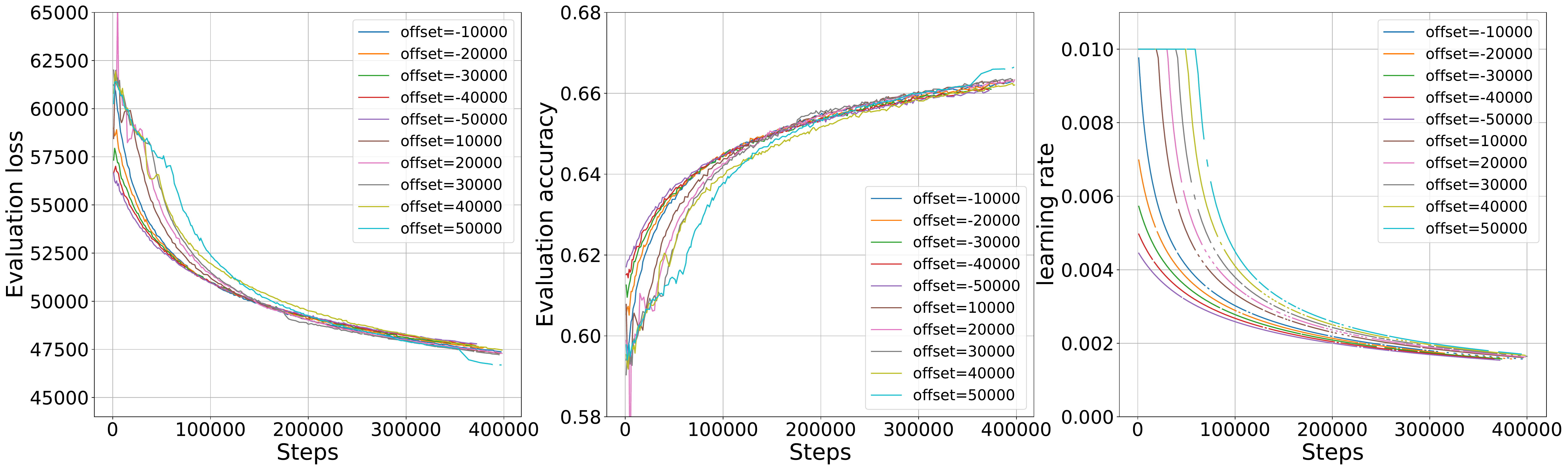}
\caption{Performance on T5 when applying offsets to the learning rate schedule.}
\label{lr_all_fig}
\end{figure*}

\nonZeroLieBracket*

\begin{proof}
Suppose we are self deep fusing a model $n$ times. Let $M$ be the manifold $M = \{ (\theta_1, \dots, \theta_n, \eta): \: \theta\in \mathbb{R}^{d_S}, \eta \in \mathbb{R}^{(d_b - d_s)}\}$; Here $d_s$, and $d_b$ are the dimension of the small and big models respectively. 
Since $\theta_i = \theta_j$ for all $i,j$, the losses $L_1$ and $L_2$ are of the form,
$$
L_1(w) =\frac 1n L_S(\theta_1) + \cdots + \frac 1n L_S(\theta_n),
$$
and,
$$
L_2(w) = L_B(\theta_1, \dots, \theta_n, \eta).
$$
The corresponding vector fields $F(w)$ and $G(w)$ are
\begin{equation*}
-\begin{pmatrix} 
\frac 1n \nabla L_S(\theta_1)  \\ 
\vdots \\ 
\frac 1n \nabla L_S(\theta_n) \\
0
\end{pmatrix} 
\ \textrm{and} \
 -\begin{pmatrix} 
\nabla_{\theta_1} L_B(\theta_1, \dots, \theta_n,\eta)  \\ 
\vdots \\ 
\nabla_{\theta_n} L_B(\theta_1, \dots, \theta_n,\eta) \\
\nabla_\eta L_B(\theta_1, \dots, \theta_n,\eta)
\end{pmatrix}.
\end{equation*}
Consider the submanifold $N\subset M $ where $$N = \{ (\theta_1, \dots, \theta_n, 0): \: \theta\in \mathbb{R}^{d_S} \ \mathrm{and}\ \theta_i = \theta_j\ \forall{i,j} \}.$$
We have the following relation between the small and big model losses on that submanifold (See Lemma \ref{lemma:same_gradient}):
\begin{equation}
% L_1(w) = L_2(w)
% \quad \textrm{and}
\quad \nabla_{\theta_i} L_1(w)= \nabla_{\theta_i} L_2(w),
\end{equation}
for all $w \in N$. Thus, the vector fields $F$ and $G$ are related as follows:
\begin{equation*}
G(w) =  -\begin{pmatrix} 
\frac 1n \nabla L_S(\theta_1)  \\ 
\vdots \\ 
\frac 1n \nabla L_S(\theta_n) \\
\nabla_\eta L_B(\theta_1, \dots, \theta_n,\eta)
\end{pmatrix} 
=   F(w)  + g_\eta(w) 
\end{equation*}
with 
$$g_\eta = \begin{pmatrix} 
0  \\ 
\vdots \\ 
0\\
-\nabla_\eta L_B(\theta_1, \dots, \theta_n, \eta)
\end{pmatrix} 
$$
This means that the Lie bracket between $F = -\nabla_w L_1$ and $G = -\nabla_w L_2$ has the following form on $N$:
\begin{align*}
[F, G] 
& =   [F,  F + g_\eta] \nonumber \\
& =   [F, F] + [F, g_\eta] \nonumber \\
& =   [F, g_\eta] \nonumber \\
& =  %- \frac 1n
\left[
\begin{pmatrix} 
\nabla L_S(\theta_1) \\ 
\vdots \\ 
\nabla L_S(\theta_n) \\
0
\end{pmatrix} , 
\begin{pmatrix} 
0 \\ 
\vdots \\ 
0 \\
\nabla_\eta L_B(\theta_1, \dots, \theta_n, \eta)
\end{pmatrix} 
\right] \nonumber \\
& =  %\frac 1n
\begin{pmatrix}
0 \\
\vdots \\
0 \\
(\sum_{i=1}^n \nabla_{\theta_i}\nabla_\eta L_B(w))\nabla L_S(\theta)
\end{pmatrix}
\end{align*}
which concludes the proof.
\end{proof}

\section{Model Dimensions for First Experiment}\label{apndx:first_experiment}
In this appendix, we list the dimension of the T5 transformers used in the first experiment.

\begin{table}[ht]
\centering
\begin{footnotesize}
\begin{tabular}{ccc}
\textbf{Model Name} & \textbf{T5-Small} & \textbf{T5-Medium} \\
\toprule
embedding dim & 512 & 1024 \\
number of heads & 6 & 12 \\
enc./dec. layers & 8 & 8 \\
head dim & 64 & 64 \\
mlp dimension & 1024 & 2048 \\
\midrule
number of parameters & 77M & 242M \\\bottomrule
\end{tabular}
\end{footnotesize}
\caption{Dimensions of T5 Small and Medium.}
\label{tab:models_dim1}
\end{table}

%\vspace{3cm}

\section{Model Dimension for Second Experiment}\label{apndx:second_experiment}
In this appendix, we list the dimension of the T5 transformers used in the second experiment.
%\vspace{-10pt}
\begin{table}[!h]
\centering
\begin{footnotesize}
\begin{tabular}{cccc}
\textbf{Model Name} & \textbf{T5-S} & \textbf{T5-M} & \textbf{T5-L} \\
\toprule
embedding dim & 512 & 1024 & 2048 \\
number of heads & 6 & 12 & 24 \\
enc./dec. layers & 8 & 8 & 8\\
head dim & 128 & 128 & 128 \\
mlp dimension & 1024 & 2048 & 4096 \\
\midrule
number of parameters & 95M & 317M & 1.1B \\\bottomrule
\end{tabular}
\end{footnotesize}
\caption{Dimensions of T5-S, T5-M and T5-L.}
\label{models_dim2}
\end{table}

\begin{figure}
    \centering
    \includegraphics[scale=0.23]{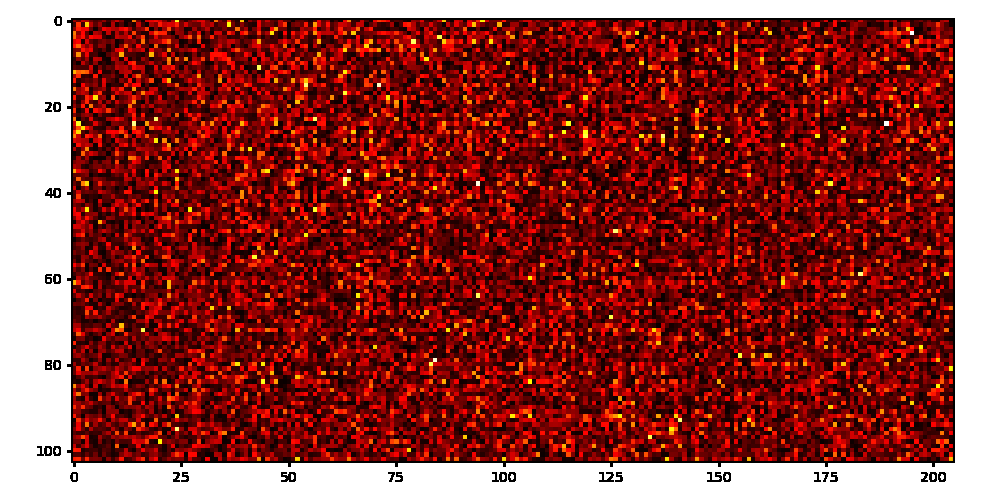}
    \caption{Self fusion parameter weight heatmap after 400K steps when using a high learning rate.}
    \label{fig:zoomin_heatmap}
\end{figure}

% Adds vertical space to force appendix C to be in the new page with the table on top

% \section{Glue Results Breakdown}\label{apndx:glue_experiment}

% In this Appendix we list the full results (see Table~\ref{tbl:full_glue}) of the downstream on various Glue tasks. The average is calculated over all tasks. For tasks with more than one metrics, we average the metrics and then average over the tasks.
\newpage
\section{Offsets Experiments Details}\label{apndx:train_dynamics}
In this appendix, we show the performance given various offsets that are applied to the learning rate schedule. The results are inline with the theoretical analysis, meaning the higher the offset (maintaining higher learning rate after growth), the better the final performance. The data is in Figure~\ref{lr_all_fig}.
%
% \begin{figure*}[t]
% \centering
% %\includegraphics[scale=0.23]{glue.png}
% \includegraphics[scale=0.18]{lr_analysis_all.pdf}
% \caption{Performance on T5 when applying offsets to the learning rate schedule.}
% \label{lr_all_fig}
% \end{figure*}

\end{document}